\documentclass[18pt]{statsoc}

\usepackage{epsfig,amsmath,amssymb,amsfonts,amstext,mathrsfs}
\usepackage{latexsym,graphics,epsf,epsfig,psfrag,color}
\usepackage{natbib}
\usepackage{graphicx}
\usepackage{epstopdf}

\newcommand{\pk}{(k)}

\newtheorem{corollary}{\textbf{Corollary}}
\newtheorem{lemma}{\textbf{Lemma}}
\newtheorem{theorem}{\textbf{Theorem}}
\newtheorem{proposition}{\textbf{Proposition}}
\newtheorem{remark}{\textbf{Remark}}

\newcommand{\nn}{\nonumber}

\newcommand{\mE}{\mathrm{E}}

\newcommand{\cR}{\mathcal{R}}

\newcommand{\cN}{\mathcal{N}}

\newcommand{\tB}{\widetilde{B}}

\newcommand{\uY}{\overrightarrow{Y}}

\newcommand{\uX}{\overrightarrow{X}}
\newcommand{\uW}{\overrightarrow{W}}

\newcommand{\ubeta}{\overrightarrow{\beta}}
\newcommand{\uZ}{\overrightarrow{Z}}
\newcommand{\uD}{\overrightarrow{D}}
\newcommand{\uU}{\overrightarrow{U}}

\newcommand{\hB}{\widehat{B}}

\newcommand{\hZ}{\widehat{Z}}

\newcommand{\hSigma}{\widehat{\Sigma}}

\newcommand{\huZ}{\widehat{\overrightarrow{Z}}}

\newcommand{\tuW}{\widetilde{\overrightarrow{W}}}

\newcommand{\hubeta}{\widehat{\overrightarrow{\beta}}}

\DeclareMathAlphabet{\matheuf}{U}{euf}{m}{n}

\newcommand{\vertiii}[1]{{\left\vert\kern-0.25ex\left\vert\kern-0.25ex\left\vert #1 \right\vert\kern-0.25ex\right\vert\kern-0.25ex\right\vert}}

\title[Multivariate Multi-Response Linear Regression via Block Regularized Lasso]{ Sharp Threshold for Multivariate Multi-Response Linear Regression via Block Regularized Lasso  }

\author[Weiguang Wang {\it et al.}]{Weiguang Wang}
\address{Syracuse University,
         Syracuse,
         U.S.A.}
\email{\{wwang23,yliang06\}@syr.edu,\;epxing@cs.cmu.edu}
\author{Yingbin Liang}
\address{Syracuse University,
         Syracuse,
         U.S.A.}
\author{Eric P. Xing}
\address{Carnegie Mellon University,
         Pittsburgh,
         U.S.A.}
\begin{document}

\maketitle

\begin{abstract}
The multivariate multi-response (MVMR) linear regression problem is investigated, in which design matrices are Gaussian with covariance matrices $\Sigma^{(1:K)}=\left( \Sigma^{(1)},\ldots,\Sigma^{(K)} \right)$ for $K$ linear regressions. The support union of $K$ $p$-dimensional regression vectors (collected as columns of matrix $B^*$) is recovered using $l_1/l_2$-regularized Lasso. Sufficient and necessary conditions on sample complexity are characterized as a sharp threshold to guarantee successful recovery of the support union. This model has been previously studied via $l_1/l_{\infty}$-regularized Lasso by \citet{Nega11} and via $l_1/l_1+l_1/l_{\infty}$-regularized Lasso by \citet{Jala10}, in which sharp threshold on sample complexity is characterized only for $K=2$ and under special conditions. In this work, using $l_1/l_2$-regularized Lasso, sharp threshold on sample complexity is characterized under only standard regularization conditions. Namely, if $n > c_{p1} \psi(B^*,\Sigma^{(1:K)})\log(p-s)$ where $c_{p1}$ is a constant, and $s$ is the size of the support set, then $l_1/l_2$-regularized Lasso correctly recovers the support union; and if $n < c_{p2} \psi(B^*,\Sigma^{(1:K)})\log(p-s)$ where $c_{p2}$ is a constant, then $l_1/l_2$-regularized Lasso fails to recover the support union. In particular, the function $\psi(B^*,\Sigma^{(1:K)})$ captures the impact of the sparsity of $K$ regression vectors and the statistical properties of the design matrices on the threshold on sample complexity. Therefore, such threshold function also demonstrates the advantages of joint support union recovery using multi-task Lasso over individual support recovery using single-task Lasso.
\end{abstract}


\section{Introduction}

Linear regression is a simple but practically very useful statistical model, in which %an dependent variable $y$ is expressed as a linear combination of a number of feature variables $\ux=(x_1,\ldots,x_p)$. More specifically,
an $n$-sample response vector $\uY$ can be modeled as
\[ \uY=X\ubeta+\uW \]
where $X\in \mathbb{R}^{n\times p}$ is the design matrix containing $n$ samples of feature vectors, $\ubeta=(\beta_1,\ldots,\beta_p)\in \mathbb{R}^p$ contains regression coefficients, and $\uW\in \mathbb{R}^n$ is the noise vector. The goal is to find the regression coefficients $\ubeta$ such that the linear relationship is as accurate as possible with regard to a certain performance criterion. The problem is more interesting in high dimensional regime with a sparse regression vector, in which the sample size $n$ can be much smaller than the dimension $p$ of the regression vector.%, but the number of nonzero regression coefficients is small compared to the dimension $p$.
%\ericx{Since you used $B$ later as the regression coefficient matrix, would it be better to use $\beta$ here?}

In order to estimate the sparse regression vector, it is natural to construct an optimization problem with an $l_0$-constraint on $\ubeta$, i.e., the number of nonzero components of $\ubeta$. However, such an optimization problem is nonconvex and in general very difficult to solve in an efficient manner as commented by \citet{Nata95}. More recently, the convex relaxation (referred to as Lasso) has been studied with an $l_1$-constraint on $\ubeta$ based on the idea in the seminal work by \citet{Tibs96}, \citet{Chen98}, and \citet{Dono01}. More specifically, the regression problem can be formulated as:
\[\min_{\ubeta\in \mathbb{R}^p} \frac{1}{n} \|\uY-X\ubeta\|_{l_2}^2+\lambda_n\|\ubeta\|_{l_1}.\]
The  $l_1$-regularized estimator has been proved by \citet{Bick09} to have similar behavior to Dantzig Selector, which was proposed by \citet{Can07}. Various efficient algorithms have been developed to solve the above convex problem efficiently (see a review monograph by \citet{Bach12}), although the objective function is not differentiable everywhere due to $l_1$-regularization. Moreover, the $l_1$-regularization is critical to force the minimizer to have sparse components as shown by \citet{Tibs96,Chen98,Dono01}.

A vast amount of recent work has studied the high dimensional linear regression problem via $l_1$-regularized Lasso under various assumptions. For example, the studies by \citet{Can06,Chen98,Elad02,Feu03,Mali04,Tro04} investigated the noiseless scenario and showed that recovery of true coefficients could be guaranteed with certain conditions on design matrices and sparsity. A number of studies focused on using $l_1$-regularization to achieve sparsity recovery for noisy scenarios. Some work (e.g., \citet{Fuch05,zhao06,Mein09}) focused on the problem with deterministic design matrices, whereas other work (e.g., \citet{Wain09,Rask10}) studied the problem with random design matrices. The work by \citet{Bach08trace} investigated linear regression model via trace norm. \citet{Tib05} and \citet{Dala12} studied linear regression model using a fusion penalty (known as the total variational penalty).

%\ericx{There is a break of flow there. I think the group lasso etc are first studied still under single response regression, but with multiple covariates panelized together via a group shrinkage (which I think is what Yuan's work about). After this, it is further generalized to multiple-response (also called multi-task) regression, where the group shrinkage was applied to one covariate to multiple tasks. Please modify the text bellow on top of my edits to make the flow consistent.}

Generalized from the $l_1$-regularized linear regression problem which aims at selecting variables individually, group Lasso is applied to regression vector $\ubeta$ in the linear regression model to select grouped variables (e.g., \citet{Yuan06,Huang10}). \citet{Jac09} and \citet{Zhao09} applied group Lasso for studying empirical risk minimization problems. \citet{bach08} studied the least square optimization problem with group Lasso.

This line of research is further generalized to block-regularization for high-dimensional multi-response (i.e., multi-task) linear regression problem, (see, e.g., \citet{Nega12} and references therein). For a multi-task regression problem, we have the following model:
\begin{equation}\label{eq:multilinear}
Y=XB^*+W
\end{equation}
where $Y \in \mathbb{R}^{n\times K}$ of which each column corresponds to the output of one task, $X\in \mathbb{R}^{n \times p}$ is the design matrix, the regression matrix $B^* \in \mathbb{R}^{p\times K}$ has each column corresponding to the regression vector for one task, and $W \in \mathbb{R}^{n\times K}$ has each column corresponding to the noise vector of one task. For each column $\uY^{(k)}$ of the matrix $Y$, it is clear that $\uY^{(k)}=X\ubeta^{*(k)}+\uW^{(k)}$, where $\ubeta^{*(k)}$ and $\uW^{(k)}$ are the corresponding columns in $B^*$ and $W$. Then each column is a single-task linear regression problem and can be solved individually. However, the $K$ individual problems (i.e., tasks) can also be coupled together via a block regularized Lasso and solved jointly in one problem. %\eric{This whole section is a little too verbose, and you can directly merge this and next paragraph to present group-regularized multi-task regression using equation 2. (Also please number all equations systematically.)}

Various types of block regularization have been proposed and studied. In the work by \citet{Oboz11}, the $l_1/l_2$-regularization was adopted to recover the support union of the regression vectors. More specifically, the following problem was studied
\begin{equation}
\min_{B \in \mathbb{R}^{p\times K}} \frac{1}{2n}\vertiii{Y-XB}_F^2+\lambda_n \|B\|_{l_1/l_2},
\end{equation}
where $\|\cdot\|_{l_a/l_b}$ is defined in ($\ref{eq:blnorm}$) in section $\ref{sec:notes}$. Sufficient and necessary conditions for correct recovery of the support union (i.e., the union of the supports of all columns of $B^*$) have been characterized.
Block regularized Lasso (as well as group Lasso) has also been applied to study various other models. For example, the $l_1/l_q$-regularized Lasso was adopted for learning structured linear regression model by \citet{Liu08}. The $l_1/l_{\infty}$-regularized Lasso was used to investigate a multi-response regression model by \citet{Tur05}. The $l_1/l_2$-regularization was used for studying empirical risk minimization problems by \citet{obo10}, and multi-task feature problems by \citet{Pon07}. The $l_1/l_q$-regularized Lasso was adopted to analyze normal means model by \citet{Kolar11}.  Blockwise sparse regression was used for studying the general loss function by \citet{Kim06}. %\udx{This paragraph is introducing the model of previous work, the following paragraph is the previous work on our own model, so we didn't combine the two paragraphs together. }

In the multi-response linear regression problem given in \eqref{eq:multilinear}, the design matrix is identical for all tasks, i.e., $X$ is the same for all column vectors of $Y$ and $B^*$. However, in many applications, it is often the case that different output variables may depend on design variables that are different or distributed differently. Thus, the resulting model includes $K$ linear regression models with different design matrices and is given by:
\begin{equation}\label{eq:ourmodel}
\uY^{(k)} = X^{(k)} \ubeta^{*(k)} + \uW^{(k)}
\end{equation}
for $k= 1,\ldots,K$, where $\uY^{(k)}\in\mathbb{R}^n$, $X^{(k)}\in\mathbb{R}^{n\times p}$, $\ubeta^{*(k)}\in\mathbb{R}^{p}$, and $\uW^{(k)}\in \mathbb{R}^n$. We refer to the above problem as the {\em multivariate multi-response (MVMR) linear regression model}, and the goal is to recover $\ubeta^{*(k)}$ for $k=1,\ldots,K$ jointly. This problem has been studied by \citet{Loun11} via the $l_1/l_2$-regularized Lasso for fixed matrices $X^{(1)},\ldots,X^{(K)}$. For random design matrices, this model has been studied via $l_1/l_{\infty}$-regularized Lasso by \citet{Nega11} and via $l_1/l_1+l_1/l_{\infty}$-regularized Lasso by \citet{Jala10} for incorporating both row sparsity and individual sparsity.

%\ericx{Overall, the above intro is a little confusing because of the unclear flow from lasso to MVMR. A general path sould be lasso, group-lasso (still single response, but your Eq. 1 already has multi-response), and then MVMR with groups on both input and output, or only on outputs. Please redo the intro unto here and make this line of development more clear. Also I think the text can be simplified, because all these are well-known standard stuff. Also, I may want to suggest use the name Structured input/output regression (SIOR) for the problems instead of MVMR, because the notion multi-variate is kind of meaningless because all practical regression is naturally multivariate, and therefore there is no point emphasizing this. With the shrinkage, you can say  Structured input/output lasso (SIRL). But if in this paper you only consider structured output, you can say structured multi-response lasso (SMRL) or structured multi-task lasso. If no structure in the output, but just an overall group penalty is used for all tasks, we can simple call it "multi-task lasso".}

In this paper, we study the MVMR problem for random design matrices via $l_1/l_2$-regularized Lasso. Although this may seem to only likely offer expected results similar to those in \citet{Oboz11}, \citet{Nega11}, and \citet{Jala10}, our exploration turns out to provide more insights which were not captured in previous studies. More detailed discussion is provided in Section 1.2. We discuss these in depth in Section $\ref{sec:compare}$. In our model, it is assumed that the design matrices are Gaussian distributed, and are independent but not identical across tasks. For each task $k$, the row vector of $X^{(k)}$ is Gaussian with mean zero and the covariance matrix $\Sigma^{(k)}$ for $k=1,\ldots,K$. The noise vectors and hence the output vectors are also Gaussian distributed and independent across tasks. We are interested in joint recovery of the union of the support sets (i.e., the support union) of regression vectors $\ubeta^{*(1)},\ldots,\ubeta^{*(K)}$. We collect these vectors together as a matrix $B^*=\left[\ubeta^{*(1)},\ldots,\ubeta^{*(K)}\right]$.

We adopt the $l_1/l_2$-regularized Lasso problem for recovery of the support union via the following optimization problem:
\begin{equation}\label{eq:optprob1}
 \min_{B\in \cR^{p\times K}} \frac{1}{2n}\sum_{k=1}^{K} \left\|\uY^{(k)}-X^{(k)} \ubeta^{(k)}\right\|_2^2 +\lambda_n \left\|B\right\|_{l_1 / l_2}
 \end{equation}
where $B=\left[ \ubeta^{(1)},\ldots,\ubeta^{(K)} \right]$. In this way, the $K$ linear regression problems are coupled together via the regularization constraint. We show that this approach is advantageous as opposed to individual recovery of the support set for each linear regression problem. This is because the $K$ regression models may share their samples in joint support recovery so that the total number of samples needed can be significantly reduced compared to performing each task individually.

%\ericx{Overall, I think the problem statement and setup can be more concise and shortened.}

\subsection{Main Contributions}

In the following, we summarize the main contributions of this work. Our results contain two parts: the achievability and the converse, corresponding respectively to sufficient and necessary conditions under which the $l_1/l_2$-regularized Lasso problem recovers the support union for the MVMR linear regression problem. Our proof adapts the techniques developed by \cite{Wain09} and by \cite{Oboz11}, but involves nontrivial development to deal with the differently distributed design matrices across tasks. This also leads to interesting generalization of the results in the paper by \cite{Oboz11} as we articulate in section $\ref{sec:compare}$.

More specifically, we show that under certain conditions that the distributions of the design matrices satisfy, if $n > c_{p1} \psi(B^*,\Sigma^{(1:K)})\log(p-s)$, where $\psi(\cdot)$ is defined in \eqref{eq:psidef} in Section $\ref{sec:notes}$ and $c_{p1}$ is a constant, then the $l_1/l_2$-regularized Lasso recovers the support union for the MVMR linear regression problem; and if $n < c_{p2} \psi(B^*,\Sigma^{(1:K)})\log(p-s)$, where $c_{p2}$ is a constant, then the $l_1/l_2$-regularized Lasso fails to recover the support union. Thus, $\psi(B^*,\Sigma^{(1:K)})\log(p-s)$ serves as a sharp threshold on the sample size.

In particular, $\psi(B^*,\Sigma^{(1:K)})$ captures the sparsity of $B^*$ and the statistical properties of the design matrices, which are important in determining the sufficient and necessary conditions for successful recovery of the support union. The property of $\psi(B^*,\Sigma^{(1:K)})$ also captures the advantages of the multi-task Lasso
%of the support union via one Lasso problem rather than
over solving each problem individually via the single-task Lasso. We show that when the $K$ tasks share the same support sets (although the design matrices can be differently distributed), $\psi(B^*,\Sigma^{(1:K)})=\frac{1}{K}\max_{1\leq k\leq K}\psi(\ubeta^*_k,\Sigma^{(k)})$. This means that the number of samples needed per task for multi-task Lasso to jointly recover the support union is reduced by $K$ compared to that of single-task Lasso to recover each support set individually.
 %This is because different tasks jointly exploit their data for recovery of the support union in the multi-task Lasso, and hence the number of samples needed per task is substantially reduced.
On the other hand, if the $K$ tasks have disjoint support sets, then $\psi(B^*,\Sigma^{(1:K)})=\max_{1 \leq k\leq K}\psi(\ubeta^{*(k)},\Sigma^{(k)})$. This implies that the number of samples needed per task to correctly recover the support union is almost the same as that of single-task Lasso to recover each support individually. Between these two extreme cases, tasks can have overlapped support sets with different overlapping levels, and the impact of these properties on the sample size for recovery of the support union is quantitatively captured by $\psi(B^*,\Sigma^{(1:K)})$.

\subsection{Comparison to Previous Results}\label{sec:compare}

The MVMR model (with differently distributed design matrices across tasks) can be viewed as generalization of the multi-response model (with an identical design matrix across tasks) studied by \cite{Oboz11}. It is thus interesting to compare our results to the results by \cite{Oboz11}. For the scenario when the tasks share the same regression vector, it is shown by \cite{Oboz11} that the major advantage of jointly solving a multi-task Lasso problem over solving each single-task Lasso problem individually is reduction of effective noise variance by the factor $K$. But the sample size needed per task for recovery of the support union via multi-task Lasso is the same as that needed for recovery of each support set individually via single-task Lasso. This implies that multi-task Lasso does not offer benefit in reducing the sample size (in the order sense) for this case. Our result, on the other hand, shows that the benefit in sample complexity by using multi-task Lasso for recovery of support union arises when the design matrices are differently distributed across tasks. For such a case, the sample size needed per task is reduced by $K$ via multi-task Lasso compared to recovery of each support set individually via single-task Lasso.
%(this is not an issue if the $K$ tasks have the same design matrix and support set, and hence the advantage of multi-task Lasso does not appear \cite{Oboz11}).
Consequently, our result is a nontrivial generalization of the result by \cite{Oboz11}. For the scenario when the tasks have disjoint support sets, our result is consistent with the result by \cite{Oboz11}, which suggests that there is no advantage of performing multi-task Lasso as opposed to performing single-task Lasso for each task.

As we mentioned before, the MVMR model has also been studied by \cite{Nega11} and \cite{Jala10}, in which $l_1/l_{\infty}$ and $l_1/l_1+l_1/l_{\infty}$-regularization were adopted for support union recovery, respectively. In these studies, sharp threshold on sample complexity is characterized only for K = 2 and under special conditions on $\frac{1}{n}X^{(k)T}_{S_k}X^{(k)}_{S_k}$. In our work, using $l_1/l_2$-regularized Lasso, we are able to characterize the sharp threshold under only standard regularization conditions.

\subsection{Relationship to Jointly Learning Multiple Markov Networks}

%For example, we may build the connections among people in social networks based on their interests. It is not hard to collect $p+1$ people's ratings for $n$ different movies and use the model $y=X\ubeta^*+w$ to solve the problem. However, we with to build the connections based on more information rather than the single area of movie rating. , where index $k$ can be interpreted as index for one specific area

%In this paper, we study the jointly Gaussian models $(\ref{eq:ourmodel})$, where we suppose all random variables satisfy the jointly Gaussian distribution.

One application of the MVMR linear regression model is to jointly learning multiple Gaussian Markov network structures. In this context, it solves a multi-task neighbor selection problem. This is also a natural scenario, in which features and their distributions vary across tasks.

%Consider a Gaussian Markov network with $p+1$ nodes represented by $X_1,\ldots,X_{p+1}$. There is no edge between the node $X_i$ and $X_j$ if these two variables are independent conditioned on all other random variables. This is equivalent to the fact that the $(i,j)^{th}$ entry of the inverse of the covariance matrix of these variables is zero.
%The nonzero entries of $B^*$ characterize the connections of random variables, therefore our model is also useful in learning Markov network structures.

We consider $K$ Gaussian Markov networks, each with ${p+1}$ nodes represented by $X_1^{(k)}$, $\ldots$ , $X_{p+1}^{(k)}$ for $k=1, \ldots, K$. The distribution of the Gaussian vector for graph $k$ is given by $\cN\left(0,\Sigma_{p+1}^{(k)}\right)$, where $\Sigma_{p+1}^{(k)}\in\mathbb{R}^{(p+1)\times(p+1)}$. Assume for each graph, there are $n$ i.i.d.\ samples generated based on the joint distribution of the nodes. The objective is to estimate the connection relationship of nodes based on the samples. We denote $n$ samples of each variable $X_j^{(k)}$ by a column vector $\uX_j^{(k)}\in \mathbb{R}^n$ for $j=1,\ldots,{p+1}$ and $k=1,\ldots,K$. For each graph $k$ and each node with index $a$, the sample vector $\uX_a^{(k)}$ can be expressed as:
\begin{equation}\label{eq:linear}
\uX_a^{(k)}=X_{-a}^{(k)} \ubeta^{(k)} + \uW_a^{(k)}
\end{equation}
where $X_{-a}^{(k)}$ is an $n \times {p}$ matrix that contains all column vectors $\uX_j^{(k)}$ for $j\neq a$, $\ubeta^{(k)}$ is a $p$-dimensional vector consisting of the estimation parameters of $X_a^{(k)}$ given $X_j^{(k)}$ with $j\neq a$, and $\uW_a^{(k)}$ is the $n$-dimensional Gaussian vector containing i.i.d.\ components with zero mean and variance given by
\begin{flalign}
&{\sigma_W^{(k)}}^2=Var(X_{1a})-Cov(X_{1a},X_{1,-a})Cov^{-1}(X_{1,-a})Cov(X_{1,-a},X_{1a}). \nn
\end{flalign}
It has been shown that the nonzero components of the vector $\ubeta^{(k)}$ represent existence of the edges between the corresponding nodes and node $a$ in graph $k$. Hence, estimation of the support set of $\ubeta^{(k)}$ provides an estimation of the graph structure, which is referred to as the {\em neighbor selection problem} by \citet{Mein06}.

%For each graph $k$, we group ${p+1}$ columns together as an $n\times {(p+1)}$ matrix $X^{(k)}=\left[\uX_1^{(k)},\ldots,\uX_{p+1}^{(k)}\right]$.

Therefore, multi-task Lasso for the MVMR linear regression problem provides an useful approach for joint neighbor selection over $K$ graphs. It is clear that in this case, the design matrices $X_{-a}^{(k)}$ in general have different distributions across $k$, and hence the MVMR model is well justified. We note that jointly learning multiple graphs has also been studied by \cite{Danaher:arXiv1111.0324} and \cite{Guo11}, which adopted a different objective function of the precision matrix $\Sigma^{-1}$. Via the MVMR linear regression model, we characterize the threshold-based sufficient and necessary conditions for joint recovery of the graphs.

\section{Problem Formulation and Notations}

%\ericx{This section is redundant to the relevant part in the intro, so the part in the intro can be simplified.}
%In multivariate linear regression problem, the design (i.e., feature) matrix is identical for all tasks, namely, it is the same corresponding to all column vectors in coefficient matrix $B^*$. However, in many applications, it is often the case that difference dependence variables may depend on different design variables $\ux$,

In this paper, we study the MVMR linear regression problem given by \eqref{eq:ourmodel}, which contains $K$ linear regressions.
%\begin{equation}\label{eq:ourmodel}
% \uY^{(k)}=X^{(k)}\ubeta^{*(k)}+\uW^{(k)}
%\end{equation}
%for $k= 1,\ldots,K$.
Here, the design matrices $X^{(1)},\ldots,X^{(K)}$ and noise vectors $\uW^{(1)},\ldots,\uW^{(K)}$ are Gaussian distributed, and are independent but not identical across $k$. For each task $k$, $X^{(k)}$ has independent and identically distributed (i.i.d.) row vectors with each being Gaussian with mean zero and covariance matrix $\Sigma^{(k)}$, and the noise vector $\uW^{(k)}$ has i.i.d. components with each being Gaussian with mean zero and variance ${\sigma_W^{(k)}}^2$. We let $\sigma_{max}=\max_{1\leq k\leq K}{\sigma_W^{(k)}}^2$.

%, i.e., is distributed as $\cN(0,\Sigma^{(k)})$

In \eqref{eq:ourmodel}, $\ubeta^{*(k)}$ denotes the true regression vector for each task $k$. We define the support set for each $\ubeta^{*(k)}$ as $S_k:=\{ j\in\{ 1,\ldots,p \}|\ubeta^{*(k)}_j\neq0 \}$. The support union over $K$ tasks is defined to be $S:=\cup^K_{k=1}S_k$. In this paper, we are interested in estimating the support union jointly for $K$ tasks.

We adopt the $l_1/l_2$-regularized Lasso to recover the support union for the MVMR linear regression model. More specifically, we solve the multi-task Lasso given in \eqref{eq:optprob1} and rewritten below:
\begin{flalign}\label{eq:optprob}
 \min_{B\in \cR^{p\times K}} \frac{1}{2n}\sum_{k=1}^{K} \left\|\uY^{(k)}-X^{(k)} \ubeta^{(k)}\right\|_2^2 +\lambda_n \left\|B\right\|_{l_1 / l_2}
\end{flalign}
where $B=\left[ \ubeta^{(1)},\ldots,\ubeta^{(K)} \right]$. In this way, the $K$ linear regression problems are coupled together via the regularization constraint. In this paper, we characterize conditions under which the solution to the above multi-task Lasso problem correctly recover the support union of the true regression vectors for $K$ tasks.

\subsection{Notations}\label{sec:notes}

We introduce some notations that we use in this paper. For a matrix $A\in\mathbb{R}^{p\times K}$, we define the $l_a/l_b$ block norm as
\begin{flalign}\label{eq:blnorm}
\left\| A \right\|_{l_a/l_b}:=\left[ \sum_{i=1}^p{\left( \sum_{j=1}^K{|A_{ij}|^b} \right)^{a/b}} \right]^{1/a}.
\end{flalign}

%which actually compute the $l_b$ norm of each row of matrix and get a vector of dimension $p$, then get the $l_a$ norm of the received vector.
We also define the operator norm for a matrix as
\[ \vertiii{A}_{a,b}:=\sup_{\|x\|_b=1}{\| Ax \|_a}. \]
In particular, we define the spectral norm as $\vertiii{A}_2=\vertiii{A}_{2,2}$ and the $l_{\infty}$-operator norm as $\vertiii{A}_{\infty}=\vertiii{A}_{\infty,\infty}=\max_{j=1,\ldots,p}\sum_{k=1}^K|A_{jk}|$, which are special cases of the operator norm.

For matrix $B=\left[ \ubeta^{(1)},\ldots,\ubeta^{(K)} \right]$ that appears in \eqref{eq:optprob}, $\ubeta^{(k)}$ denotes its $k$th columns for $k=1,\ldots,K$. We further let $B_i$ to be the $i$th row of $B$. Similarly, for $B^*=\left[ \ubeta^{*(1)},\ldots,\ubeta^{*(K)} \right]$ that contains true regression vectors, its $k$th column is denoted by $\ubeta^{*(k)}$ and the $i$th row is denoted by $B^*_i$. We next define the normalized row vectors of $B^*$ as
\[Z^*_i=\begin{cases}
\frac{B_i^*}{\left\|B_i^*\right\|_{l_2}} \quad & \text{if } B_i^* \neq 0 \\
0 & \text{otherwise},
\end{cases} \]
and define the matrix $Z^*$ to contain $Z^*_i$ as its $i$th row for $i=1,\ldots,p$. To avoid confusion, we use $\hB$ to denote the solution to the multi-task Lasso problem \eqref{eq:optprob}.

%By using $true$ $coefficient$ $matrix$, we are actually saying
%\[ \uY^{(k)}=X^{(k)}\ubeta^{*(k)}+\uW^{(k)} \],
%where $\uW^{(k)}$ is the $n$-dimensional zero mean Gaussian vector with i.i.d components.

The support union $S(B)$ for a matrix $B\in\mathbb{R}^{p\times K}$ is denoted as $S(B)=\{i\in \{1,\ldots,p\}| B_i \neq 0\}$, which includes indices of the nonzero rows of the matrix $B$. We use $S$ to represent $S(B^*)$ (i.e., the true support union) for convenience and use $S^c$ to denote the complement of the set $S$. We let $s=|S|$ denote the size of the set $S$. For any matrix $X^{(k)}\in\mathbb{R}^{n\times p}$, the matrix $X_S^{(k)}$ contains the columns of matrix $X^{(k)}$ with column indices in the set $S$, and $X_{S^c}^{(k)}$ contains the columns of matrix $X^{(k)}$ with column indices in the set $S^c$. Similarly, $B_S^*$ and $Z_S^*$ respectively contain rows of $B^*$ and $Z^*$ with indices in $S$.

As each row of matrix $X^{(k)}$ is Gaussian distributed as $\cN(0,\Sigma^{(k)})$, we use $\Sigma_{SS}^{(k)}$ to denote the covariance matrix for each row of $X_S^{(k)}$, and use $\Sigma_{S^c S}^{(k)}$ to denote the cross covariance between rows of $X_{S^c}^{(k)}$ and $X_S^{(k)}$.

For convenience, we use $\Sigma^{(1:K)}$ to denote a set of matrices $\Sigma^{(1)},\ldots,\Sigma^{(K)}$. We also define the following functions of matrices $Q^{(1:K)}$ to simplify our notations:
\begin{flalign}
& \rho_u\left(Q^{(1:K)}\right):=\max_{j\in S^c} \max_{1 \leq k\leq K} Q_{jj}^{(k)}, \nn \\
& \rho_l\left(Q^{(1:K)}\right):=\min_{i,j\in S^c, j\neq i} \min_{1 \leq k\leq K} \left[Q_{jj}^{(k)}+Q_{ii}^{(k)}-2Q_{ji}^{(k)}\right]. \nn
\end{flalign}
In particular, our results contain the functions $\rho_u\left(\Sigma^{(1:K)}_{S^cS^c|S}\right)$ and $\rho_l\left(\Sigma^{(1:K)}_{S^cS^c|S}\right)$, where $\Sigma^{(k)}_{S^cS^c|S}$ is the covariance matrix of each row of $X_{S^c}^{(k)}$ with $X_S^{(k)}$ given.

For matrix $B^*$, we define $b_{min}^*=\min_{j\in S}\left\|B_j^{*}\right\|_{l_2}$. We define the following function that captures the sparsity of $B^*$ and the statistical properties of the design matrices $X_S^{(1:K)}$:
\begin{flalign}\label{eq:psidef}
\psi(B^*,\Sigma^{(1:K)}):=\max_{1 \leq k\leq K}{\uZ_{Sk}^{*T}{\left(\Sigma_{SS}^{(k)}\right)}^{-1}\uZ_{Sk}^{*}},
\end{flalign}
where $\uZ_{Sk}^{*}$ is the $k$th column of $Z_S^*$. We note that this definition of $\psi(\cdot)$ function is different from the previous work by \cite{Oboz11} with the same design matrix for all tasks. Here, due to different design matrices across the $K$ tasks, $\psi(\cdot)$ depends on $K$ quantities $\uZ_{Sk}^{*T}{\left(\Sigma_{SS}^{(k)}\right)}^{-1}\uZ_{Sk}^{*}$ with each depending on a column vector $\uZ_{Sk}^*$.

We denote $g(\cdot)=o\left( f(\cdot) \right)$ if $\lim_{n\to \infty}\frac{g(\cdot)}{f(\cdot)}\to 0$, and $g(\cdot)=O\left( f(\cdot) \right)$ if $\lim_{n\to \infty}\frac{g(\cdot)}{f(\cdot)}\to c_o$, where the constant $0<c_o< \infty$.

%For the noise vectors $\uW^{(k)}$, $\sigma_{max}$ denotes the maximum noise variance across all tasks, $\sigma_{max}=\max_{1\leq k\leq K}{\sigma_W^{(k)}}^2$.

%$\psi(B^*)$ is an important function, since it measure the sparsity of matrix $B^*$. Later we will show this quantity has the same order of $s=|S(B^*)|$.

\section{Main Results}

In this section, we provide our main results on using $l_1/l_2$-regularized Lasso to recover the support union for the MVMR linear regression model. Our results contain two parts: one is the achievability, i.e., sufficient conditions for the $l_1/l_2$-regularized Lasso to recover the support union; and the other is the converse, i.e., conditions under which the $l_1/l_2$-regularized Lasso fails to recover the support union. We then discuss implications of our results by considering a few representative scenarios, and compare our results with those for the multivariate linear regression with an identical design matrix across tasks.

\subsection{Achievability and Converse}

We first introduce a number of conditions on covariance matrices $\Sigma^{(k)}$ for $k=1,\ldots,K$, which are useful for the statements of our results.

(C1). There exists a real number $\gamma \in (0,1]$ such that $ \vertiii{A}_{\infty} \leq 1-\gamma$, where $A_{js}=\max_{1\leq k\leq K} \left|\left( \Sigma_{S^cS}^{\pk} \left(\Sigma_{SS}^{\pk} \right)^{-1} \right)_{js} \right|$ for $j \in S^c$ and $s \in S$.

(C2). There exist constants $0<C_{min}\leq C_{max}<+\infty$ such that all eigenvalues of the matrix $\Sigma_{SS}^{(k)}$ are contained in the interval $\left[C_{min}, C_{max}\right]$ for $k=1,\ldots,K$.

(C3). There exists a constant $D_{max}<+\infty$ such that $\max_{1\leq k\leq K} \vertiii{ \left(\Sigma_{SS}^{\pk} \right)^{-1} }_{\infty} \leq D_{max}$.

In this paper, we consider the asymptotic regime, in which $p \rightarrow \infty$, $s \rightarrow \infty$, and $\log{(p-s)}\to +\infty$. In such a regime, we introduce the conditions on the regularization parameter and the sample size $n$ as follows:

(P1). Regularization parameter $\lambda_n=\sqrt{\frac{f(p)\log p}{n}}$, where the function $f(p)$ is chosen such that $f(p)\to +\infty$ as $p\to +\infty$, and $\frac{f(p)\log p}{n} \rightarrow 0$ as $n \rightarrow \infty$, i.e., $\lambda_n\to 0$ as $n \to +\infty$.

(P2). Define $\rho(n, s, \lambda_n)$ as
\[ \rho(n, s, \lambda_n):=\sqrt{\frac{8\sigma_{max}^2s\log{s}}{nC_{min}}}+\lambda_n\left( D_{max}+\frac{12s}{C_{min}\sqrt{n}}\right) \]
and require $ \frac{\rho(n, s, \lambda_n)}{b_{min}^*}=o(1)$.

The following theorem characterizes sufficient conditions for recovery of the support union via $l_1/l_2$-regularized Lasso.
\begin{theorem}\label{th:achieve}
Consider the MVMR problem in the asymptotic regime, in which $p\to \infty$, $s\to \infty$ and $\log (p-s) \to \infty$. We assume that the parameters $\left(n, p, s, B^*,\Sigma^{(1:K)}\right)$ satisfy the conditions (C1)-(C3), and (P1)-(P2). If for some small constant $v>0$,
\begin{equation}\label{eq:largen}
n>2(1+v) \psi\left(B^*,\Sigma^{(1:K)}\right)\log(p-s)\frac{\rho_u\left(\Sigma^{(1:K)}_{S^cS^c|S}\right)}{\gamma^2} \; ,
\end{equation}
then the multi-task Lasso problem \eqref{eq:optprob} has a unique solution $\hB$, the support union $S(\hB)$ is the same as the true support union $S(B^*)$,\;and $\| \hB-B^* \|_{l_{\infty}/l_2}=o(b^*_{min})$ with the probability greater than
\begin{flalign}
 1-K\exp{(-c_0\log{s})}-\exp{(-c_1\log{(p-s)})}
\end{flalign}
where $c_0$ and $c_1$ are constants.
\end{theorem}
%\begin{flalign}
% 1&-2(K+1)\exp\left(-\frac{s}{2}\right)-4(K+1)\exp\left(-\frac{n}{2}\left(\frac{1}{4}-\sqrt{\frac{s}{n}}\right)_+^2\right) \nn \\
%&-K\exp\left( \log{s}-2\log{s}\left[ 1-2\sqrt{\frac{1}{2\log{s}}} \right] \right) \nn \\
%&-\exp\left( -\frac{v}{2}\log{(p-s)} \right)-\exp{\left( -5(n-s)\left[ 1-2\sqrt{\frac{1}{5}} \right] \right)}.
%\end{flalign}

%It is clear from the above theorem that under certain conditions, the multi-task Lasso problem has a unique solution with high probability, and such a solution correctly recovers the support union of the MFMV linear regression model. In particular, to guarantee correct recovery, the sample size must be larger than the order of $s\log(p-s)$, where  $\psi(B^*,\Sigma^{(1:K)})$ contributes the order $s$. This agrees with previous studies of Lasso problems \cite{ }.

Theorem \ref{th:achieve} provides sufficient conditions on the sample size such that the solution to the $l_1/l_2$-regularized Lasso problem correctly recovers the support union of the MVMR linear regression model. We next provides a theorem about the conditions on the sample size under which the solution to the $l_1/l_2$-regularized Lasso problem fails to recover the support union.
\begin{theorem}\label{th:converse}
Consider the MVMR problem in the asymptotic regime, in which $p\to \infty$, $s\to \infty$ and $\log (p-s) \to \infty$. We assume that the parameters $\left(n, p, s, B^*, \Sigma^{(1:K)}\right)$ satisfy the conditions (C1)-(C2) and the conditions: $s/n=o(1)$ and $\frac{1}{\lambda_n^2 s}\to 0$. If for some small constant $v>0$,
\begin{equation}
n<2(1-v)\psi(B^*,\Sigma^{(1:K)})\log{(p-s)}\frac{\rho_l\left(\Sigma_{(S^cS^c|S)}^{(1:K)}\right)}{(2-\gamma)^2},
\end{equation}
then with the probability greater than
\begin{equation}
1-\exp(-c_2 s)-c_3\exp \left( -c_4 \frac{n}{s} \right)
\end{equation}
for some positive constants $c_2,c_3$ and $c_4$, no solution $\hB$ to the multi-task Lasso problem \eqref{eq:optprob} recovers the true support union and achieves $\| \hB-B^* \|_{l_{\infty}/l_2}=o(b^*_{min})$.
\end{theorem}

The proofs of Theorems \ref{th:achieve} and \ref{th:converse} are provided in Section $\ref{sec:proofachieve}$ and $\ref{sec:proofconverse}$, respectively.
Combining Theorems \ref{th:achieve} and \ref{th:converse}, it is clear that the quantity
$\psi(B^*,\Sigma^{(1:K)})\log(p-s)$ serves as a threshold on the sample size $n$, which is tight in the order sense. As the sample size is above the threshold, the multi-task Lasso recovers the true support union, and as the sample size is below the threshold, the multi-task Lasso fails to recover the true support union. The following proposition provides bounds on the scaling behavior of the function $\psi(B^*,\Sigma^{(1:K)})$ in the asymptotic regime.
\begin{proposition}\label{pro:psiorder}
Consider the MVMR linear regression model with the regression matrix $B^*$ and the covariance matrices $\Sigma^{(1:K)}$ satisfying the condition (C2), the function $\psi(B^*,\Sigma^{(1:K)})$ is bounded as
\[\frac{s}{KC_{min}} \leq \psi(B^*,\Sigma^{(1:K)})\leq \frac{s}{C_{min}}. \]
\end{proposition}

The proof of the Proposition $\ref{pro:psiorder}$ is provided in Appendix $\ref{app:psi}$.

In the next subsection, we explore the properties of the quantity $\psi(B^*,\Sigma^{(1:K)})$ in order to understand the impact of sparsity of $B^*$ and covariance matrices $\Sigma^{(1:K)}$ on sample complexity for recovering the support union.

%The theorem tells us given the conditions $(1)-(6)$,  And the solution is unique with high probability. Note that the second part of the statements requires the support union of the solution for problem \eqref{eq:optprob} to recover the true support union of the Gaussian model. And the true support union gives us information for all edges. According the theorem, it is reasonable to estimate the connection relationship in $K$ graphs by solving the constructed multi-task lasso problem. Furthermore, given the conditions for parameters $(n, p, s, \Sigma^{(k)})$, the correction rate for recovery converge to $1$ with rate as large as $O(\min{\{ K\exp{(-c_0\log{s})},(p-s)\exp{(-c_1\log{(p-s)})} \}})$.

%We would like to illustrate the advantages of our work by showing examples. The basic idea of our main result is to recover the support sets for $K$ graphs using one single lasso problem. One of the benefits is sample size of the jointly recover process will be much less than solving graphs individually. Since the main result shows the sample size $n$ is proportional to function $\psi(B^*)$, we analyze the $\psi(B^*)$ for multi-task problem and $\psi(\ubeta^*_k)$ for a single graph. Here $\ubeta^*_k$ is the $k$th column of $B^*$ and $\ubeta^*_k=\ubeta^{(k)}$.

\subsection{Implications}

The quantity $\psi(B^*,\Sigma^{(1:K)})$ captures sparsity of $B^*$ and statistical properties of design matrices $\Sigma^{(1:K)}$, and hence plays an important role in determining the conditions on the sample size for recovery of the support union as shown in Theorems \ref{th:achieve} and \ref{th:converse}. In this section, we analyze $\psi(B^*,\Sigma^{(1:K)})$ for a number of representative cases in order to understand advantages of multi-task Lasso which solves multiple linear regression problems jointly over single-task Lasso which solves each linear regression problem individually.

We denote $\psi(\ubeta^{*(k)},\Sigma^{(k)})$ as the function corresponding to a single linear regression problem, where $\ubeta^{*(k)}$ represents the $k$th column of $B^*$. It is clear that $\psi(\ubeta^{*(k)},\Sigma^{(1:K)})$ captures the threshold on the sample size for the single-task Lasso problem. Comparison of $\psi(B^*,\Sigma^{(1:K)})$ and $\psi(\ubeta^{*(k)},\Sigma^{(k)})$ provides comparison between multi-task Lasso and single-task Lasso in terms of the number of samples needed for recovery of the support union/set.
We explicitly express $\psi(B^*,\Sigma^{(1:K)})$ and $\psi(\ubeta^{*(k)},\Sigma^{(k)})$ as follows:
\begin{flalign}
& \psi (B^*,\Sigma^{(1:K)}) =\max_{1 \leq k\leq K}\sum_{i \in S}\sum_{j \in S}\frac{B_{ik}^* B_{jk}^*}{\left\| B_i^* \right\|_{l_2} \left\| B_j^* \right\|_{l_2}}\left( \left(\Sigma_{SS}^{(k)}\right)^{-1}\right)_{ij}
\end{flalign}
\begin{flalign}
\psi(\ubeta^{*(k)},\Sigma^{(k)})=\sum_{i \in S}\sum_{j \in S}\frac{\ubeta^{*(k)}_i \ubeta^{*(k)}_j}{\left| \ubeta^{*(k)}_i \right|\left| \ubeta^{*(k)}_j \right|}\left( \left(\Sigma_{SS}^{(k)}\right)^{-1}\right)_{ij}
\end{flalign}
where $B^*_{ik}$ denotes the $(i,k)$th entry of the matrix $B^*$ and $\ubeta^{*(k)}_i$ denotes the $i$th entry of the vector $\ubeta^{*(k)}$. %The matrix $\left(\Sigma_{SS}^{(k)}\right)^{-1}$ contains the correlation information among $X_S^{(k)}$.

%The rest of this section focuses on the comparison between $\psi(B^*)$ and $\psi(\ubeta^*_k)$. We consider some special cases of the Gaussian model. We will compare our results with some of the previous work to show some advantages of our work.

We first study the scenario, in which all $K$ tasks have the same regression vectors, and hence have the same support sets.
%Also, what if they have different $\beta$ (value-wise), but same support?}.

\begin{corollary}\label{cor:ex1} (Identical Regression Vectors)
If $B^*$ has identical column vectors, i.e., ${\ubeta^{*(k)}}=\ubeta^*$ for $k=1,\ldots,K$, then
\begin{flalign}
\psi(B^*,\Sigma^{(1:K)})&=\frac{1}{K}\max_{1\leq k\leq K}\psi(\ubeta^*,\Sigma^{(k)}).
\end{flalign}
\end{corollary}

\begin{proof}
Under the assumption of the corollary,
$B^*=\ubeta^*\vec{1}_K^T$, where $\ubeta^*\in\mathbb{R}^p$. Hence, $\uZ_{Sk}^{*}=\frac{sign(\ubeta_S^*)}{\sqrt{K}}$, where the vector $\ubeta_S^*$ contains components in the support $S$.
\begin{flalign}
\psi (B^*,\Sigma^{(1:K)})&=\max_{1 \leq k\leq K}{\uZ_{Sk}^{*T}{\left(\Sigma_{SS}^{(k)}\right)}^{-1}\uZ_{Sk}^{*}} \nn \\
&=\max_{1 \leq k\leq K}{\frac{{sign(\ubeta_S^*)}^T}{\sqrt{K}}{\left(\Sigma_{SS}^{(k)}\right)}^{-1}\frac{sign(\ubeta_S^*)}{\sqrt{K}}} \nn \\
%&=\frac{1}{K}\max_{1\leq k\leq K}{sign(\ubeta_S^*)}^T{\left(\Sigma_{SS}^{(k)}\right)}^{-1}sign(\ubeta_S^*) \nn \\
&=\frac{1}{K}\max_{1\leq k\leq K}\psi(\ubeta^*,\Sigma^{(k)}) .
\end{flalign}
\end{proof}

\begin{remark}
Corollary \ref{cor:ex1} implies that the number of samples per task needed to correctly recover the support union via multi-task Lasso is reduced by a factor of $K$ compared to single-task Lasso that recovers each support set individually.
\end{remark}

%Consider the requirement about sample number $n$ of each graph in $\eqref{eq:largen}$. Under the assumption of identical structures, $\psi(B^*,\Sigma^{(1:K)})$ is the only term related to the number $K$ of tasks. If we compute $K$ tasks jointly instead of computing each task individually, the above equation shows that $\psi(B^*,\Sigma^{(1:K)})=\frac{1}{K}\max_{1\leq k\leq K}\psi(\ubeta^*_k,\Sigma^{(k)})$. This implies that the sample number of each graph is reduced to $1/K$.

%Especially, if we add one more condition $\Sigma_{SS}^{(k)}=\Sigma_{SS}$ for all $k=1,\ldots,K$. This makes the $K$ graphs to be totally identical in the sense that both the edges between $X_j^{(k)}$ and $Y^{(k)}$ and the edges within $X_S^{(k)}$ are identical. We have $\psi(B^*)=\frac{1}{K}\psi(\ubeta^*)$ for this case.

It can be seen that although the $K$ tasks involve design matrices that have different covariances, as long as dependence of the output variables on the feature variables is the same for all tasks, the tasks share samples in multi-task Lasso to recover the support union so that the sample size needed per task is reduced by a factor of $K$. Hence, there is a significant advantage of grouping tasks with similar regression vectors together for multi-task learning.

Corollary \ref{cor:ex1} can be viewed as a generalization of the result by \cite{Oboz11}, in which the design matrices for the tasks are the same. The result by \cite{Oboz11} suggests that if the tasks share the same regression vector, there is no benefit in terms of the number of samples needed for support recovery using multi-task Lasso compared to single-task Lasso. Our result suggests that the benefit of multi-task Lasso in fact arises when the design matrices are differently distributed. For such a case, we show that the sample size needed per design matrix (i.e., per task) is reduced by the factor $K$. %\ericx{This seems a little counter intuitive. Can you add some high level discussion, and explanation of the insight of why? (Can I view in the following way: multi-task with same $\beta$ on same features is like a single task problem, and why design is different, it means we have samples with good representativeness, therefore we need fewer samples. But if we have sample with the same design, no power is gained? Still, I found this intuition not sounding good to myself.)}

Moreover, compared to recovery of each support set individually via single-task Lasso,
%\ericx{This seems to be an unfinished sentence.}
multi-task Lasso also reduces sample size per task by the factor $K$. However, such an advantage does not appear if the $K$ tasks have the same design matrix and regression vectors as by \cite{Oboz11}.

We next study a more general case when regression vectors are also different across tasks (but the support sets of tasks are the same) in addition to varying design matrices across tasks.
\begin{corollary}\label{cor:ex2} (Varying Regression Vectors with Same Supports)
Suppose all entries $B_{jk}^*>0$ for $j\in S$ and $k=1,\ldots,K$, and all coefficients are bounded, i.e., $\bar{B_k}-\Delta_k\leq B_{jk}^* \leq \bar{B_k}+\Delta_k$, where $\Delta_k>0$ is a small perturbation constant with $\bar{B}_k>\Delta_k$. Then,
\[ \frac{\psi(B^*,\Sigma^{(1:K)})}{\max_{1 \leq k\leq K}\psi(\ubeta^{*(k)},\Sigma^{(k)})}\leq \frac{1}{K}\max_{1 \leq k\leq K}\frac{\left(\bar{B_k}+\Delta_k\right)^2}{\left(\bar{B_k}-\Delta_k\right)^2}. \]
%Hence, for small enough $\max_{1 \leq k \leq K} \Delta_k$, the sample size needed for recovery of the support union via the multi-task Lasso is reduced by a factor $K$ compared to the single-task Lasso.
\end{corollary}
\begin{proof}
Based on the assumption for $B^*$, we obtain the following upper bound on $\psi(B^*,\Sigma^{(1:K)})$ and lower bound on $\psi(\ubeta^{*(k)},\Sigma^{(k)})$:
\begin{flalign}
& \psi(B^*,\Sigma^{(1:K)}) \leq \frac{1}{K}\max_{1 \leq k\leq K} \frac{1}{\left( \bar{B_k}-\Delta_k\right)^2 } \sum_{i \in S}\sum_{j \in S} B_{ik}^* B_{jk}^* \left( \left(\Sigma_{SS}^{(k)}\right)^{-1}\right)_{ij};
\end{flalign}
%and
%\[ \psi(B^*)\geq\frac{1}{K}\max_{1 \leq k\leq K} \frac{1}{\left( \bar{B_k}+\Delta_k\right)^2 } \sum_{i \in S}\sum_{j \in S} B_{ik}^* B_{jk}^* \left( \left(\Sigma_{SS}^{(k)}\right)^{-1}\right)_{ij}.  \]
%We also obtain the following lower bound on $\psi(\ubeta^{*(k)},\Sigma^{(k)})$:
%\[ \psi(\ubeta^*_k)\leq \frac{1}{\left( \bar{B_k}-\Delta_k\right)^2 } \sum_{i \in S}\sum_{j \in S} B_{ik}^* B_{jk}^* \left( \left(\Sigma_{SS}^{(k)}\right)^{-1}\right)_{ij}\]
\begin{flalign}
&\psi(\ubeta^{*(k)},\Sigma^{(k)}) \geq \frac{1}{\left( \bar{B_k}+\Delta_k\right)^2 } \sum_{i \in S}\sum_{j \in S} B_{ik}^* B_{jk}^* \left( \left(\Sigma_{SS}^{(k)}\right)^{-1}\right)_{ij}.
\end{flalign}
%Combining the above bounds, we obtain
%\[ \frac{1}{K}\max_{1 \leq k\leq K} \frac{1}{\left( \bar{B_k}-\Delta_k\right)^2 } \sum_{i \in S}\sum_{j \in S} B_{ik}^* B_{jk}^* \left( \left(\Sigma_{SS}^{(k)}\right)^{-1}\right)_{ij} \]
Combining the above bounds, we obtain
\[ \frac{\psi(B^*,\Sigma^{(1:K)})}{\max_{1 \leq k\leq K}\psi(\ubeta^{*(k)},\Sigma^{(k)})}\leq \frac{1}{K}\max_{1 \leq k\leq K}\frac{\left(\bar{B_k}+\Delta_k\right)^2}{\left(\bar{B_k}-\Delta_k\right)^2}\;. \]

\end{proof}

%Hence, for small enough $\max_{1 \leq k \leq K} \Delta_k$, the sample size needed for multi-task Lasso is still reduced by the factor of $K$ compared to single-task Lasso.
%\end{proof}

Corollary \ref{cor:ex2} is a strengthened version of Corollary \ref{cor:ex1} in that Corollary \ref{cor:ex2} allows both the regression vectors and design matrices to be different across tasks and still shows that the number of samples needed is reduced by a factor of $K$ compared to single-task Lasso, as long as the support sets across tasks are the same.

\begin{corollary}\label{cor:ex3} (Disjoint Support Sets)
Suppose the distribution of all design matrices are the same, i.e., $\Sigma^{(k)}=\Sigma$ for $k=1,\ldots,K$, and suppose that the support sets $S_k$ of all tasks are disjoint. Let $s_k=|S_k|$, and hence $s=\sum_{k=1}^K s_k$. Then,
\[ \psi(B^*,\Sigma^{(1:K)})=\max_{1 \leq k\leq K}\psi(\ubeta^{*(k)},\Sigma^{(k)}), \]
where $\Sigma^{(1:K)}=\left( \Sigma,\ldots,\Sigma \right)$.
%and the sample size $n$ for recovery of the support union via multi-task Lasso is reduced by a factor of at least $\frac{\log(p-s)}{\log(p-\max_{k}s_k)}$.
\end{corollary}

\begin{proof}
By the assumption of the corollary, % $B^* =\left[ \ubeta^{*(1)},\ldots,\ubeta^{*(K)} \right]$, and $\uZ_{Sk}^*=sign\left( \ubeta_S^{*(k)} \right)$.
we obtain:
\begin{flalign}
\psi(B^*,\Sigma^{(1:K)}) & =\max_{1 \leq k\leq K}{sign\left( \ubeta_S^{*(k)} \right)}^T\Sigma_{SS}^{-1}\; sign\left( \ubeta_S^{*(k)} \right) =\max_{1 \leq k\leq K} \psi(\ubeta^{*(k)},\Sigma^{(k)}).
%\psi(\ubeta^{*(k)},\Sigma^{(k)}) & ={sign\left( \ubeta_S^{*(k)} \right)}^T\Sigma_{SS}^{-1}sign\left( \ubeta_S^{*(k)} \right). \nn
\end{flalign}
\end{proof}
\vspace{-3mm}
We note that
\begin{flalign}
\max_{1 \leq k\leq K} & \psi(\ubeta^{*(k)},\Sigma^{(k)})\log{(p-s)} \leq\max_{1 \leq k\leq K}\psi(\ubeta^{*(k)},\Sigma^{(k)})\log{(p-s_k)}. \nn
\end{flalign}

Since the number of samples needed per task for multi-task Lasso is proportional to $\max_{1 \leq k \leq K} \psi (\ubeta^{*(k)},\Sigma^{(k)}) \log{(p-s)}$, and the number of samples needed for single-task Lasso for task $k$ is proportional to $\psi(\ubeta^{*(k)},\Sigma^{(k)})\log{(p-s_k)}$, the above equation implies that the required number of samples for multi-task Lasso is smaller than (in fact almost the same as) that for single-task Lasso.

Corollary \ref{cor:ex3} suggests that if the tasks have disjoint support sets for regression vectors, the advantage of the multi-task Lasso vanishes. This is reasonable because the tasks do not benefit from sharing the samples for recovering the supports if their support sets are disjoint. The essential message of Corollary \ref{cor:ex3} should not change if the tasks have different design matrices and/or different regression vectors. The critical assumption in Corollary \ref{cor:ex3} is the disjoint support sets.

Corollaries \ref{cor:ex1} and \ref{cor:ex3} provide two extreme cases when the tasks share the same support sets and have disjoint support sets, respectively. The number of samples needed per task for recovery of the support union goes from $1/K$ of to the same as the sample size needed for single-task Lasso. Between these two extreme cases, tasks may have overlapped support sets with various overlapping levels. Correspondingly, the number of samples needed for recovering the support union should depend on the overlapping levels of the support sets and is captured precisely by the quantity $\psi(B^*,\Sigma^{(1:K)})$. We demonstrate such behavior via our numerically results in the next section.

\section{Numerical Results}

In this section, we provide numerical simulations to demonstrate our theoretical results on using block-regularized multi-task Lasso for recovery of the support union for the MVMR linear regression model. We study how the sample size needed for correct recovery of the support union depends on sparsity of the regression vectors, on the distributions of the design matrices, and on the number of tasks.

We first study the scenario considered in Corollary $\ref{cor:ex1}$ when the $K$ tasks have the same regression vectors, i.e., $B^*=\ubeta^*\vec{1}_K^T$. We set $\ubeta^*=\frac{1}{\sqrt{K}}\vec{1}_S$, where $S$ is the common support set across $K$ tasks. We set the covariance matrix $\Sigma^{(k)}$ to be different across $K$ tasks as follows. For $k=1,\ldots,K$, we set $Cov(X_a,X_b)>0$ (where $a,b \in \{1,2,\ldots,p\}$) if $a=b\pm 1$, and otherwise $Cov(X_a,X_b)=0$. In particular, $Cov(X_a,X_b)=1+1/k$ if $a=b\pm 1$ and $a$ is odd, and $Cov(X_a,X_b)=1-0.8/k$ if $a=b\pm 1$ and $a$ is even. The sparsity of linear regression vectors is linearly proportional to the dimension $p$, i.e., $s=\alpha p$, with the parameter $\alpha$ controlling the sparsity of the model. We set $\alpha=1/8$. We choose the dimension $p=128,256,512$. We set the regularization parameter $\lambda_n=3.5 \times\sqrt{\log{(p-s)}\log s/n}$. We solve the $l_1/l_2$-regularized multi-task Lasso problem \eqref{eq:optprob} for recovery of the support union for $K=2, 4, 6, 8$.

\begin{figure}[ht!]
\vspace{-0.2cm}
\begin{center}
\begin{tabular}{ccc}
\includegraphics[width=6.0cm]{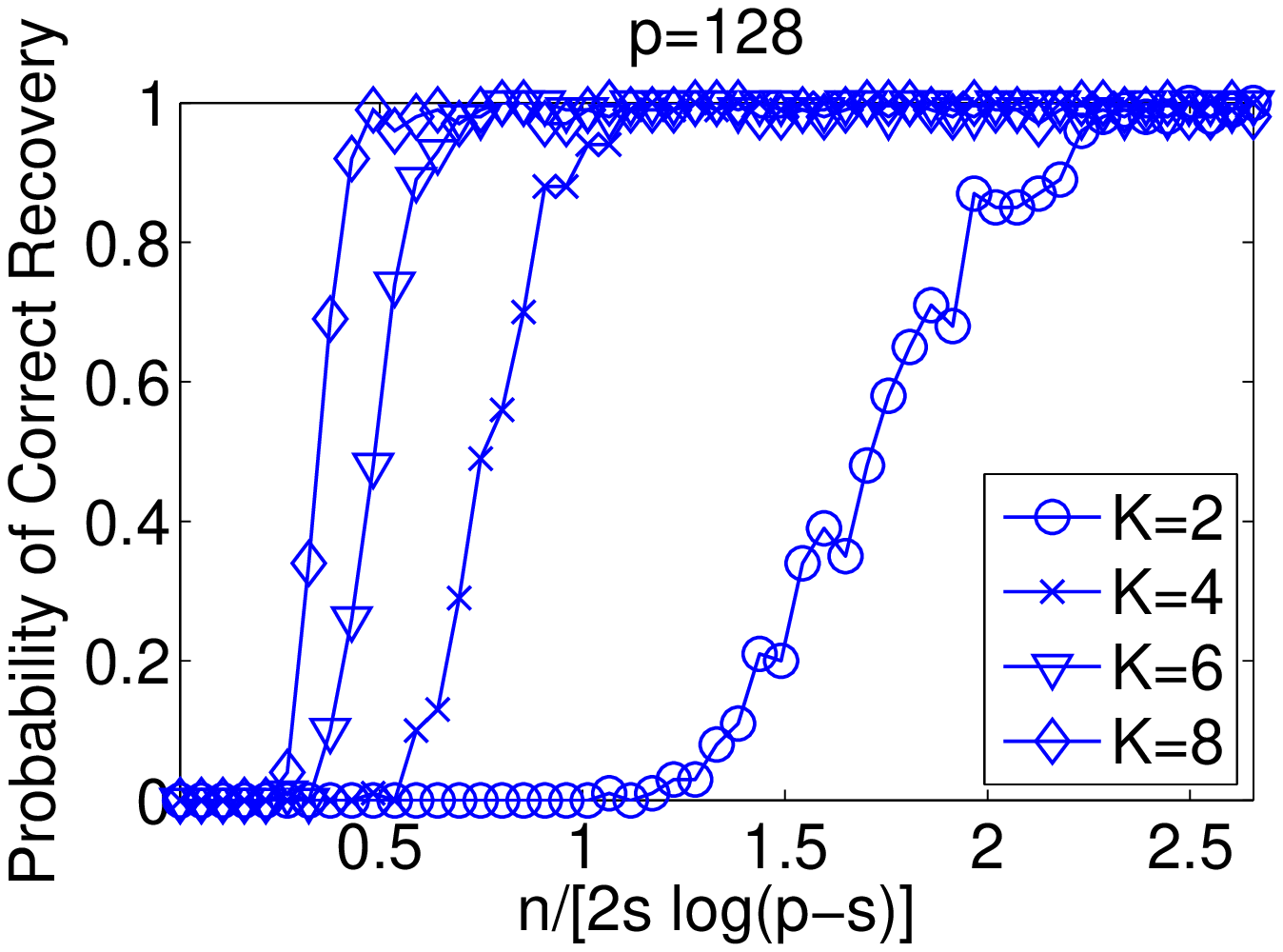} &
\includegraphics[width=6.0cm]{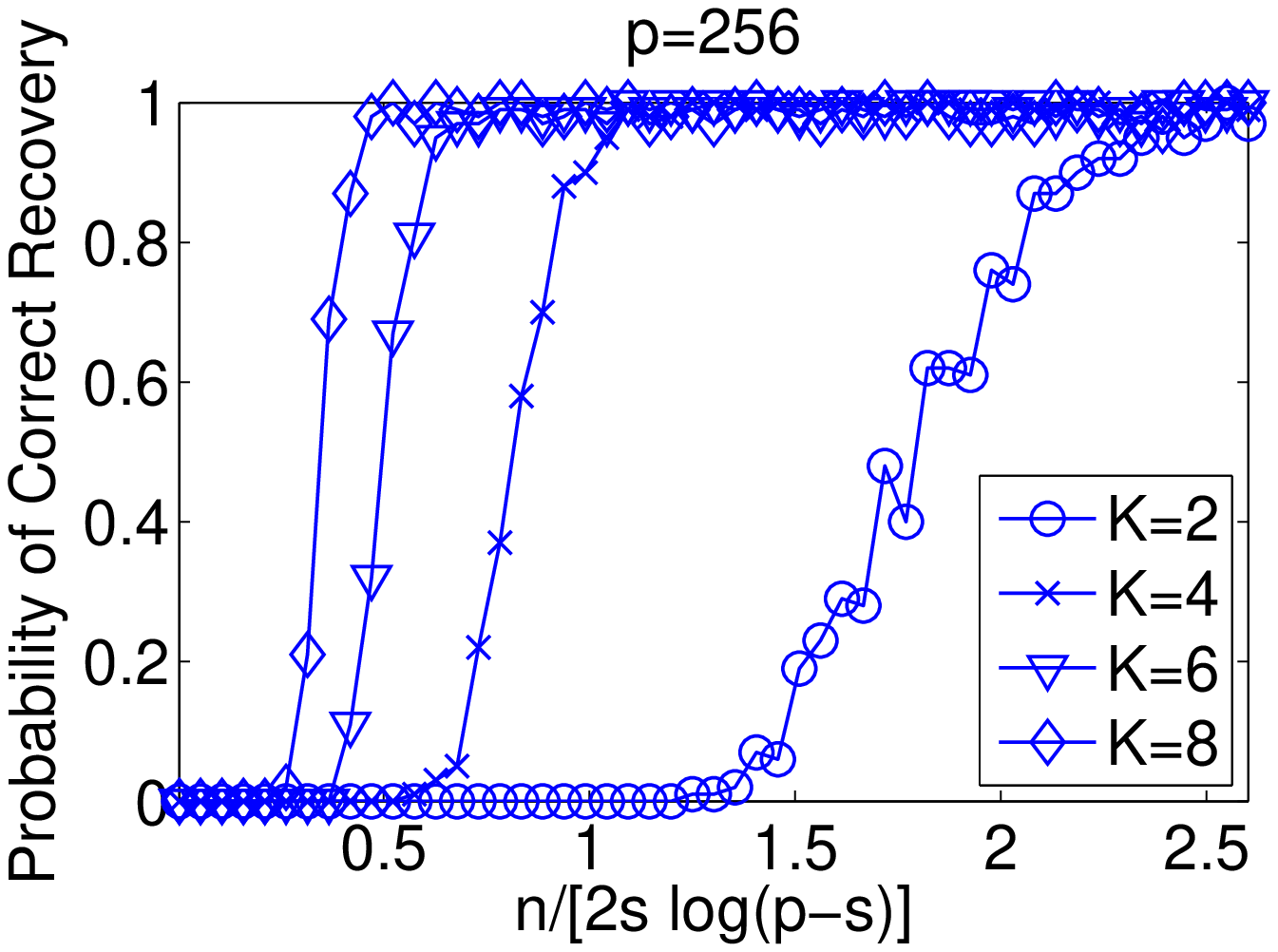}
\end{tabular}
\vspace{-0.2cm}
\includegraphics[width=6.0cm]{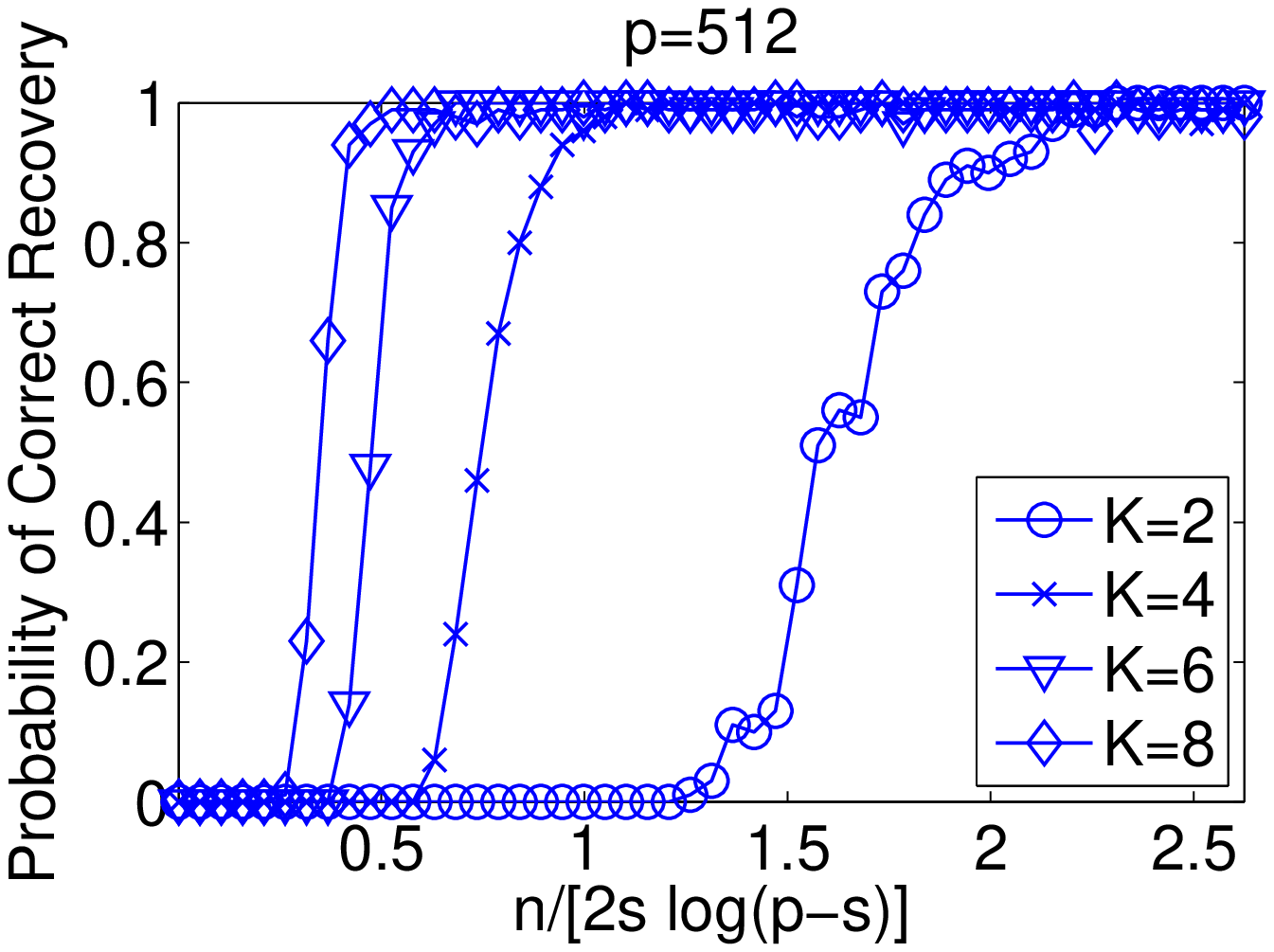}
\end{center}
\vspace{-0.2cm}
\caption{Impact of number of tasks on the sample size for scenarios with identical regression vectors and varying distributions for design matrices across tasks}
\label{fig:diversesigma}
\end{figure}

\begin{figure}[ht!]
\vspace{-0.2cm}
\begin{center}
\begin{tabular}{ccc}
\includegraphics[width=6.0cm]{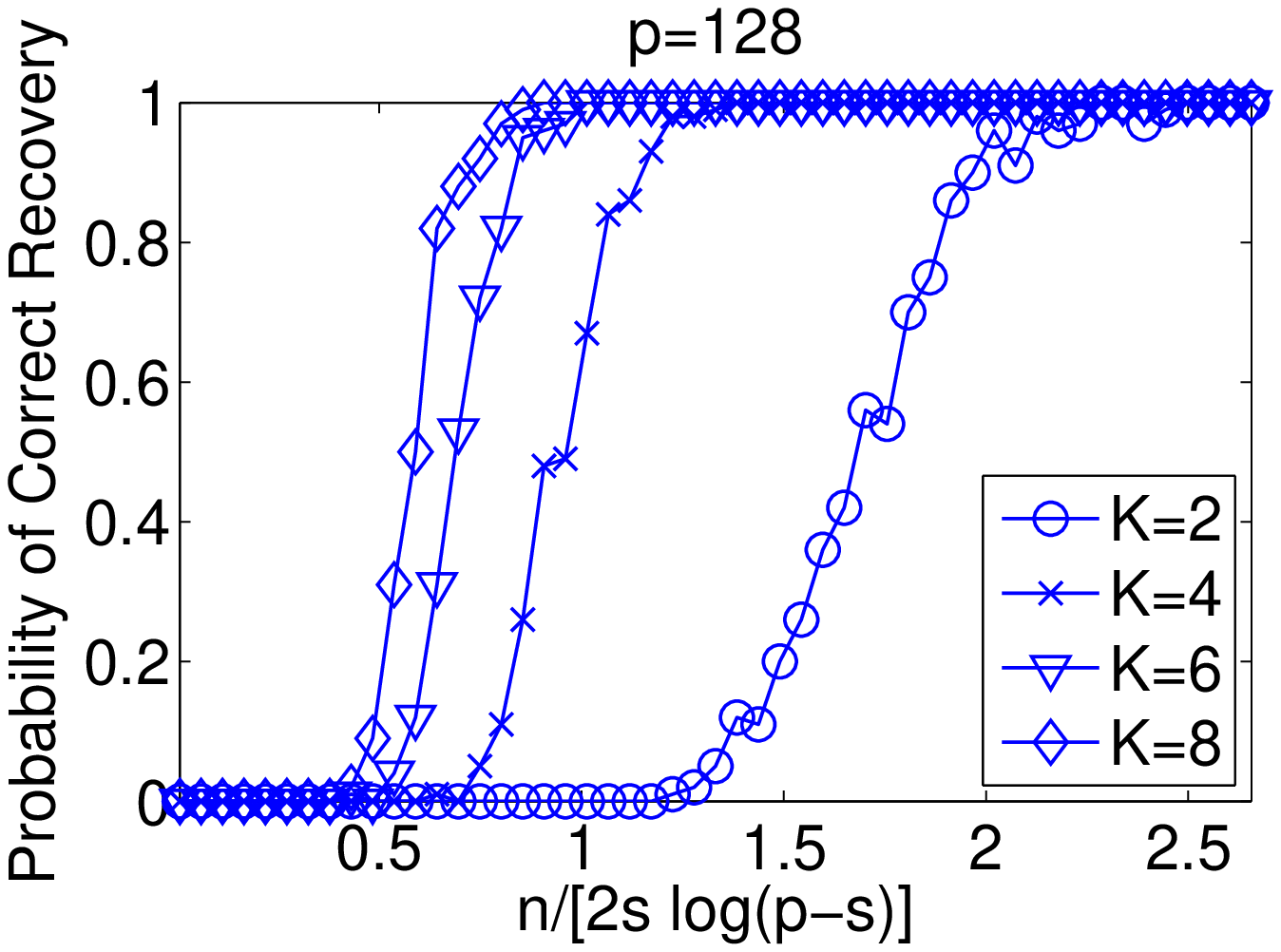} &
\includegraphics[width=6.0cm]{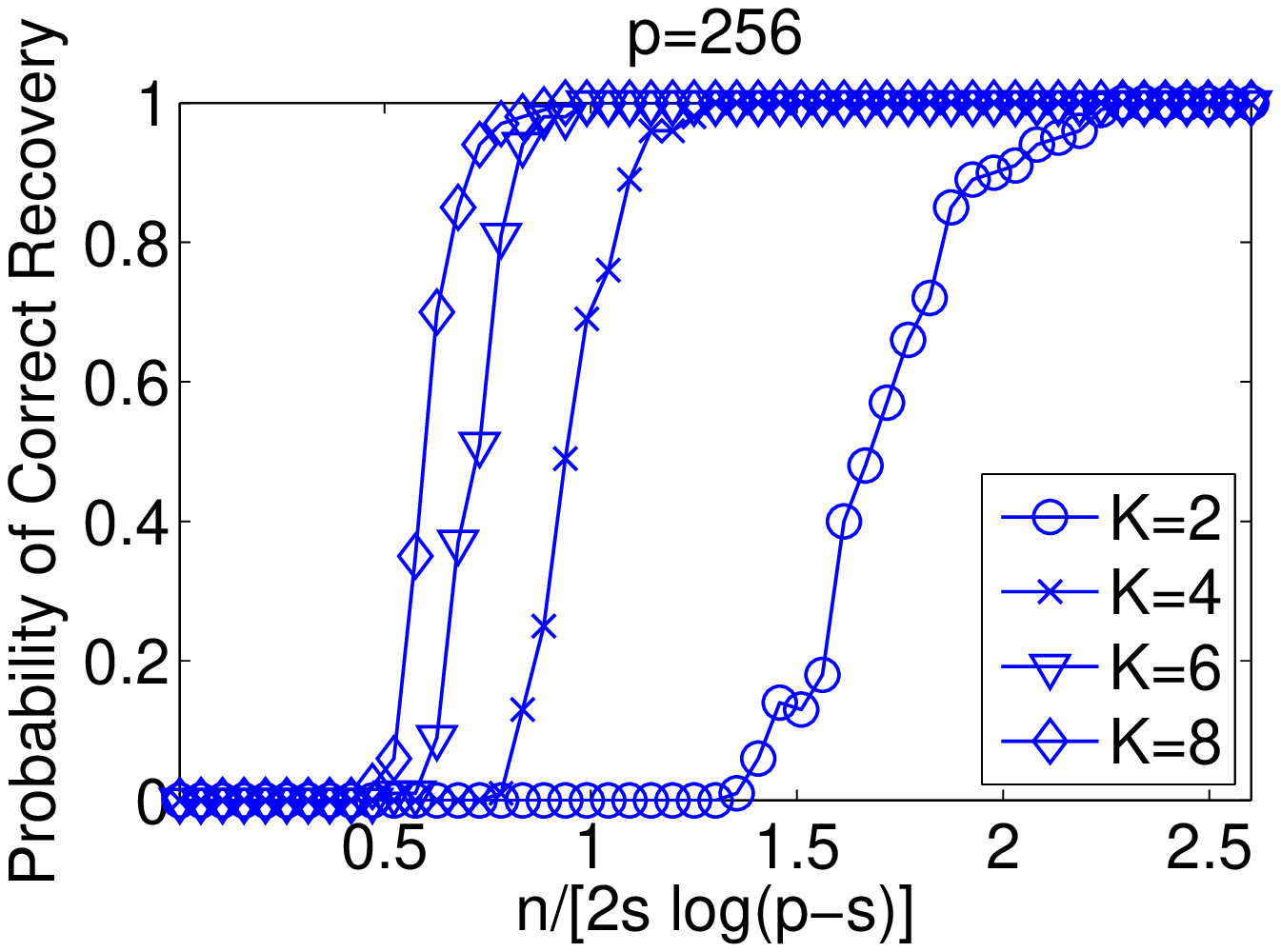}
%\\
%$p=128$ & $p=256$ & $p=512$
\end{tabular}
\vspace{-0.2cm}
\includegraphics[width=6.0cm]{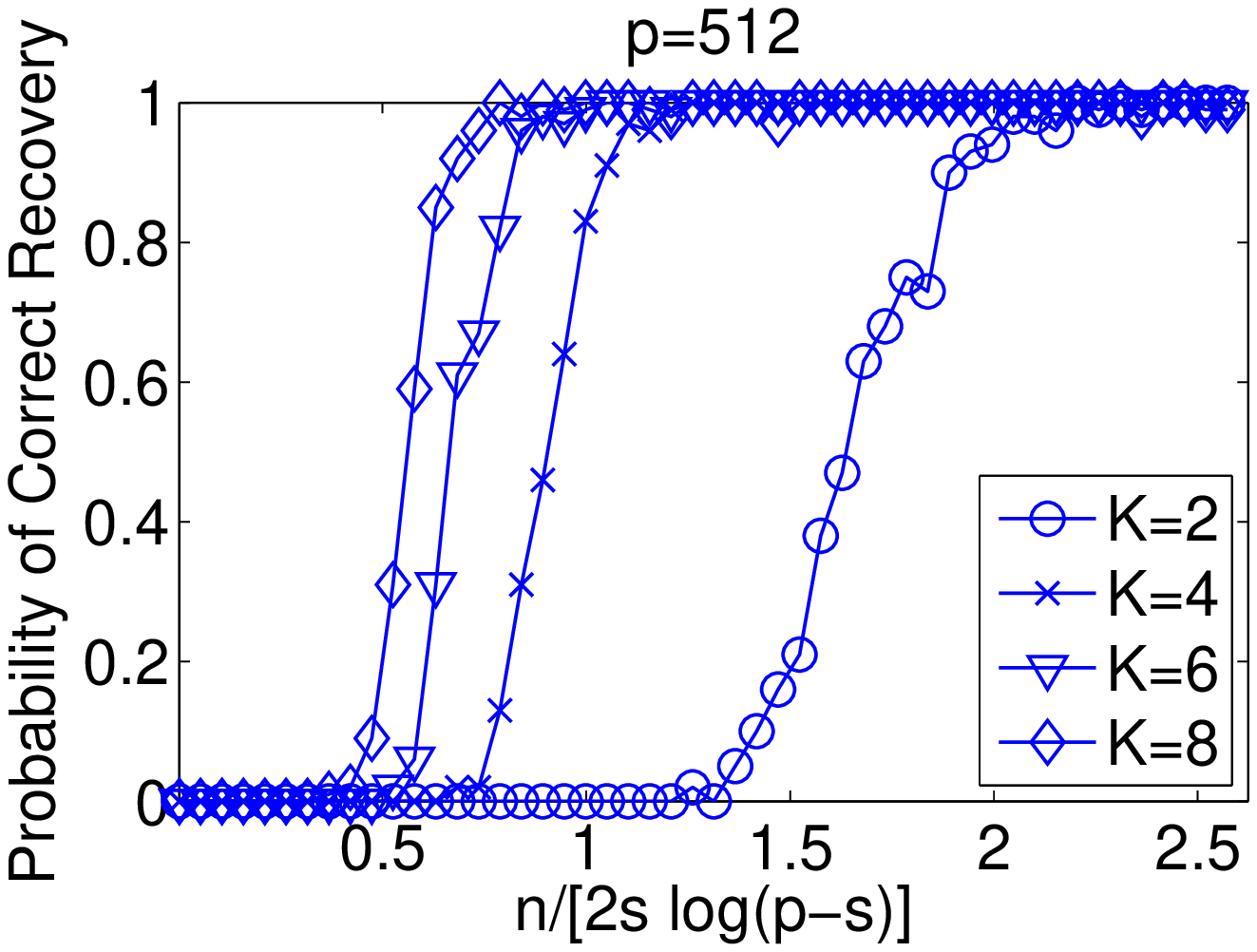}
\end{center}
\vspace{-0.2cm}
\caption{Impact of number of tasks on the sample size for scenarios with non-equal regression values and identical design matrix distribution across tasks}
\label{fig:diverse}
\end{figure}

Fig.~\ref{fig:diversesigma} plots the probability of correct recovery of the support union as a function of the scaled sample size. It can be seen that the sample size for guaranteeing correct recovery scales in the order of $s\log(p-s)$ for all plots. Moreover, as the number of tasks $K$ increases, the sample size (per task) needed for correct recovery decreases inversely proportionally with $K$, which is consistent with Corollary $\ref{cor:ex1}$. These results demonstrate that when the regression vectors are the same across tasks, multi-task Lasso has a great advantage compared to single-task Lasso in terms of reduction in the sample size needed per task.

%Although this numerical result shows the advantage of multi-task Lasso over the single -task Lasso, the model conditions are restricted since both the regression vectors and the distributions of design matrices are identical across all the tasks. As the analysis of corollary $\ref{cor:ex1}$ shows, the sample scale per task is reduced by the task number $K$ even if the design matrices are differently distributed across tasks. Therefore, we propose a modified numerical experiment and the model conditions are changed to a more general case in which the design matrices distributions are different. Other parameters ($B^*, n, p, s, \lambda$) are kept the same as the previous simulation. Fig.~\ref{fig:diversesigma} shows the result of this experiment, which is quite similar to the case with identical covariance matrices. However, it still can be noticed that the sample size needed for varying covariance matrices is a little larger than that for identical design matrix distribution. Even though, the result still satisfies the corollary $\ref{cor:ex1}$ since sample scale per task is proportional to $1/K$ with differently distribution design matrices.

We are also interested in the influence of non-equal regression values on the sample size for correct recovery. Our next experiment is taken for the scenario in which all tasks share the same support sets but have non-equal regression values across tasks. For $k=1,\ldots,K$, $\ubeta^{*(k)}_j=\frac{1}{\sqrt{K}}\times \left( 1+\frac{k}{16} \right)$ for $j=16t_{pe}$, and $\ubeta^{*(k)}_j=\frac{1}{\sqrt{K}}\times \left( 1-\frac{k}{16} \right)$  for $j=16t_{pe}+8$, where $t_{pe}$ is any nonnegative integer such that $j\leq p$. The covariance matrices $\Sigma^{(k)}$ are set to be identical across all tasks. We set $Cov(X_a,X_b)=1$ (where $a,b \in \{1,2,\ldots,p\}$) if $a=b\pm 1$, and otherwise $Cov(X_a,X_b)=0$. Other parameters are chosen to be the same as the experiment in Fig.~\ref{fig:diversesigma}. Fig.~\ref{fig:diverse} plots how the probability of correct recovery changes with the sample size for $p=128,256,512$. It exhibits the same behavior as Fig.~\ref{fig:diversesigma}, although now the regression vectors have unequal values across tasks. In particular, it can be seen that the sample size needed for correct recovery decreases as the number of tasks increased, demonstrating the advantage of multi-task Lasso.

We next study how the overlapping levels of the support sets across tasks affect the sample size for correct recovery of the support union. We set $K=2$, i.e., two tasks, and study three overlapping models for the two tasks: (1) same support sets $S_1=S_2=\{ j\leq p: 8t_{pe}+1,  \;where\; integer\; t_{pe}\geq0 \}$; (2) disjoint support sets $S_1\bigcap S_2=\phi$ in which $S_1=\{ j\leq p: 16t_{pe}+1,  \;where\; integer\; t_{pe}\geq0  \}$ and $S_2=\{ j\leq p: 16t_{pe}+2,  \;where\; integer\; t_{pe}\geq0  \}$; (3) overlapping support sets in which $S_1=\{j\leq p: j=24t_{pe}+1$ or $j= 24t_{pe}+2\,  \;where\; integer\; t_{pe}\geq0 \}$, and $S_2=\{j\leq p: j=24t_{pe}+2$ or $j= 24t_{pe}+3\,  \;where\; integer\; t_{pe}\geq0 \}$.
%\begin{color}{red} describe how the two are overlapped\end{color}.
We choose the linear sparsity model with $\alpha=1/8$. We set $p=128,256,512$, and $\Sigma^{(k)}=I_p$ for $k=1$ and $2$. We also set $\lambda_n=3.5 \times \sqrt{\log{(p-s)}\log s/n}$.

Fig.~\ref{fig:overlap} compares the probability of correct recovery as a function of the scaled sample size for the three overlapping models. It can be seen that the model with the same support set requires the smallest sample size, and the model with disjoint support sets requires the largest sample size. The model with overlapping support sets needs the sample size between the two extreme models. This is reasonable because as the support sets overlap more, tasks share more information in samples for support recovery and hence need less number of samples for correct recovery.

\begin{figure}[ht!]
\vspace{-0.2cm}
\begin{center}
\begin{tabular}{ccc}
\includegraphics[width=6.0cm]{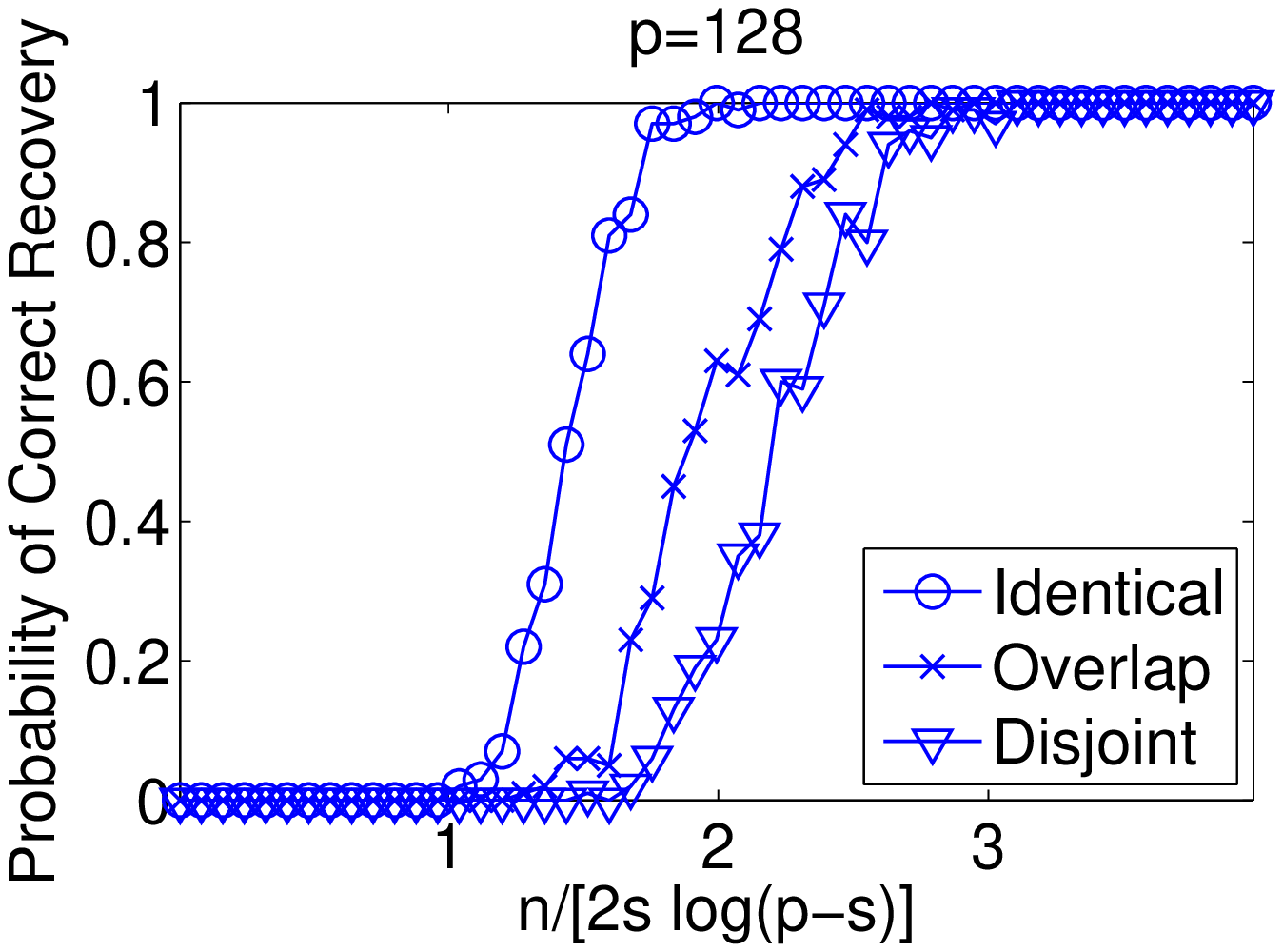} &
\includegraphics[width=6.0cm]{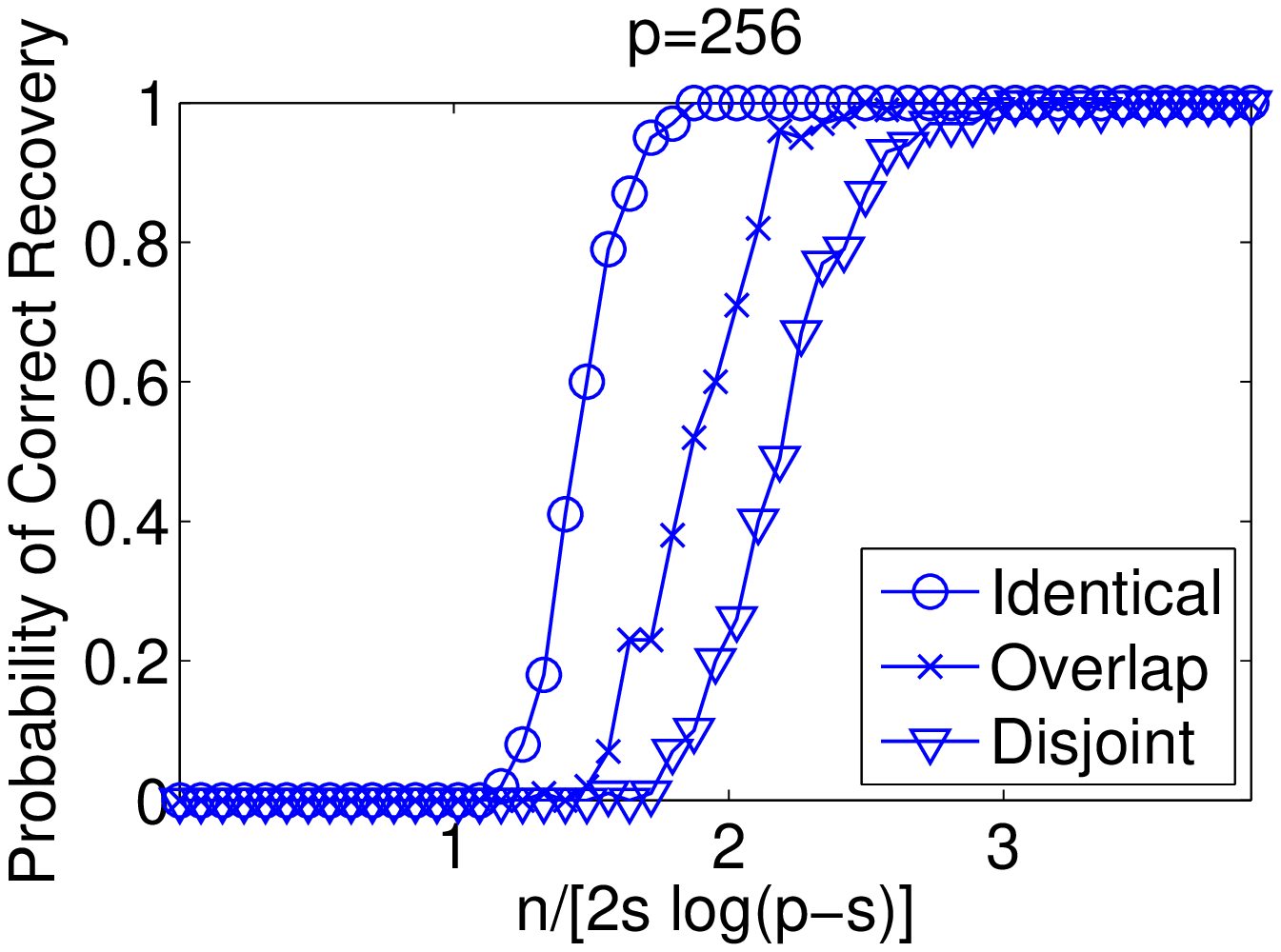}
\end{tabular}
\vspace{-0.2cm}
\includegraphics[width=6.0cm]{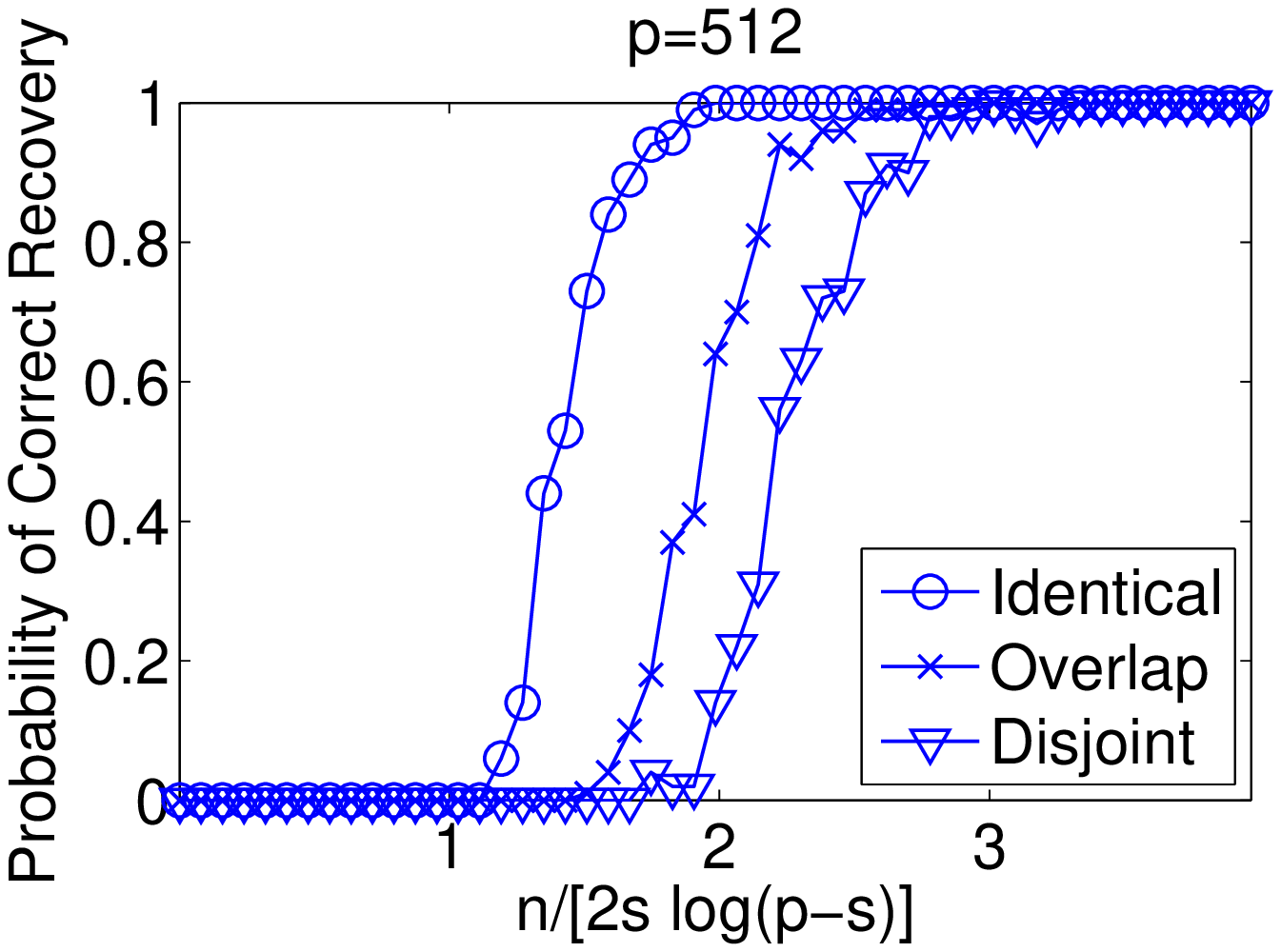}
\end{center}
\vspace{-0.2cm}
\caption{Impact of overlapping levels of support sets on the sample size with same regression values for overlapping entries and identical distributions for design matrices across tasks}
\label{fig:overlap}
\end{figure}

\begin{figure}[ht!]
\vspace{-0.2cm}
\begin{center}
\begin{tabular}{ccc}
\includegraphics[width=6.0cm]{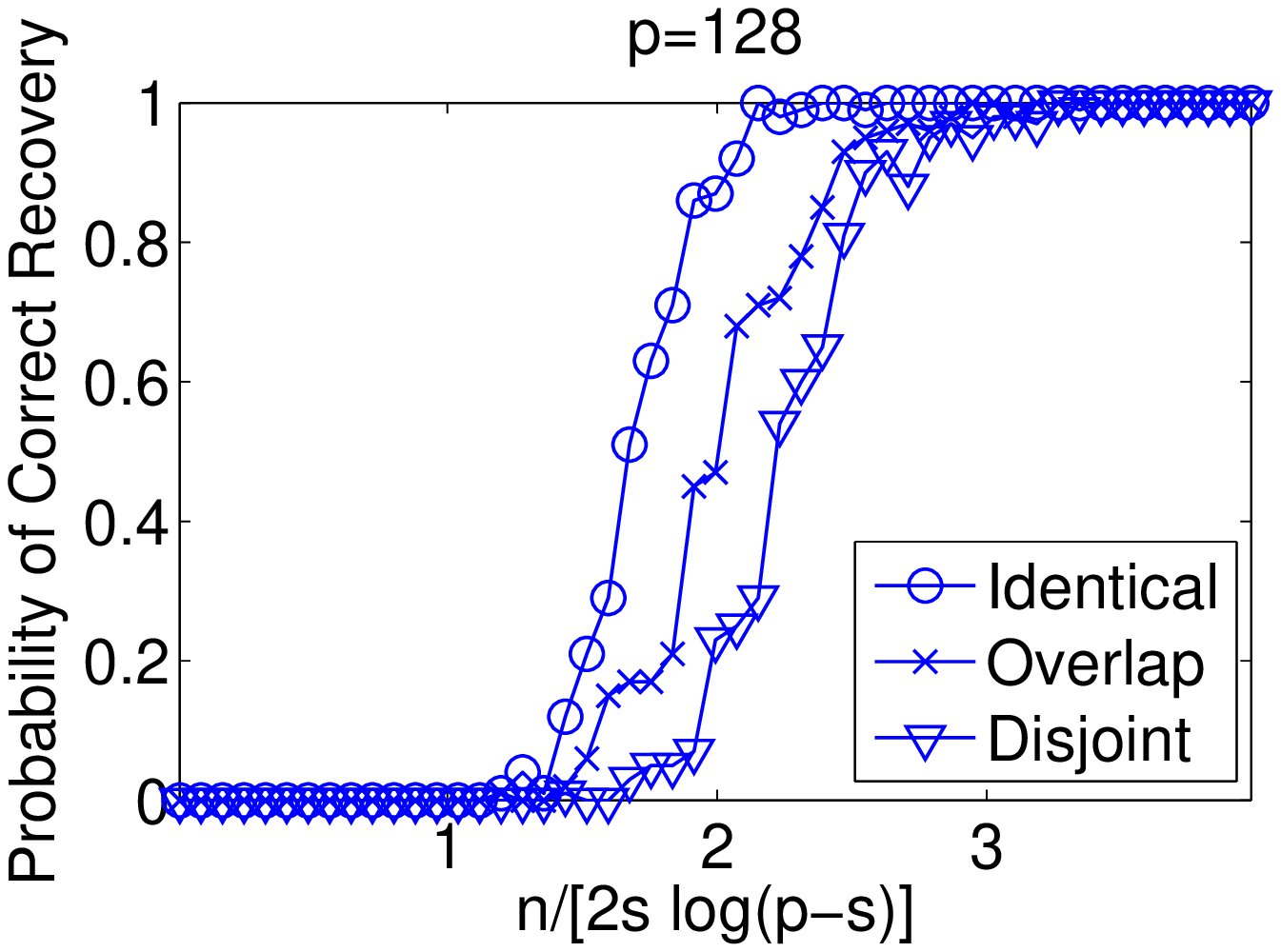} &
\includegraphics[width=6.0cm]{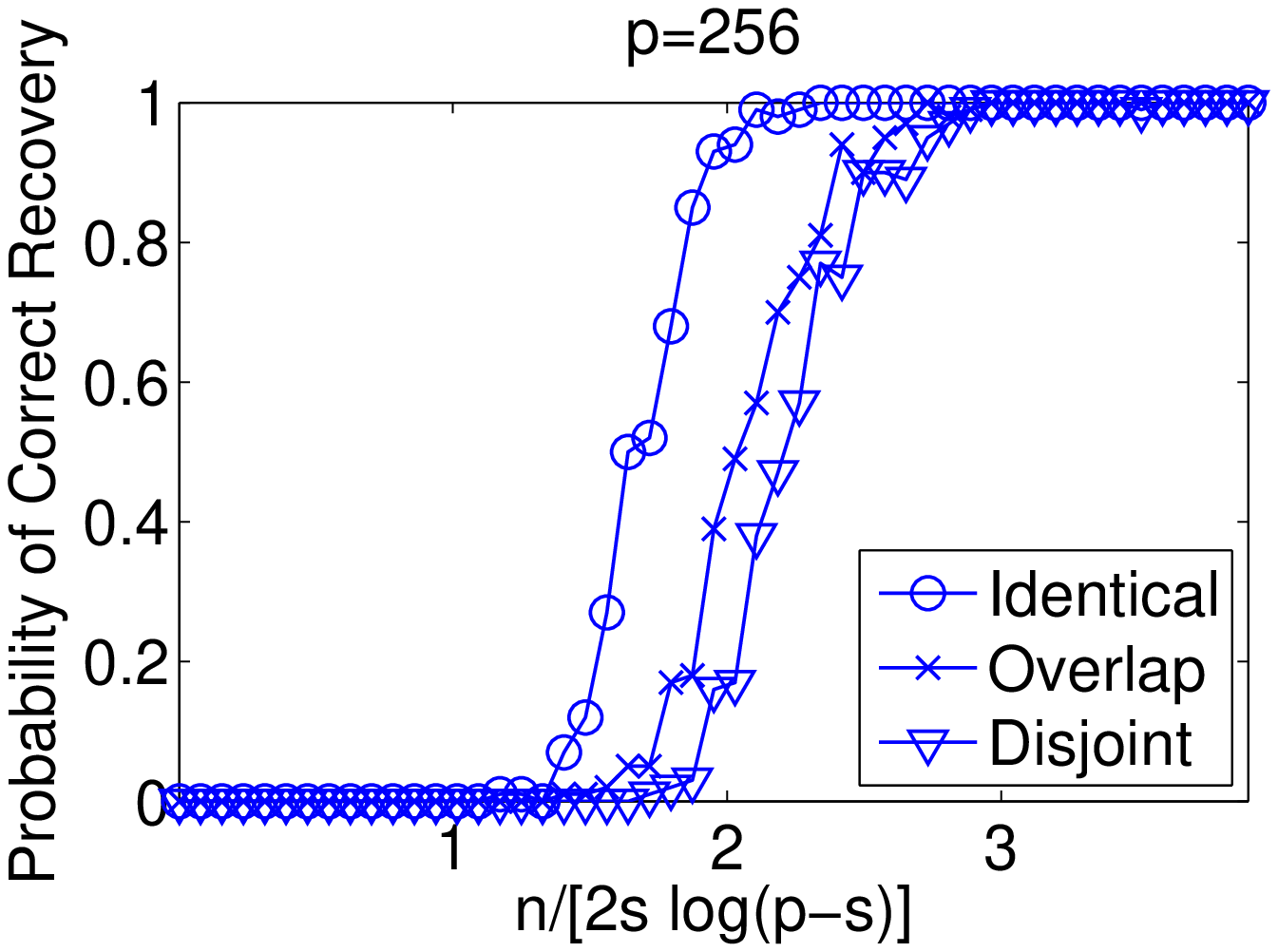}
\end{tabular}
\vspace{-0.2cm}
\includegraphics[width=6.0cm]{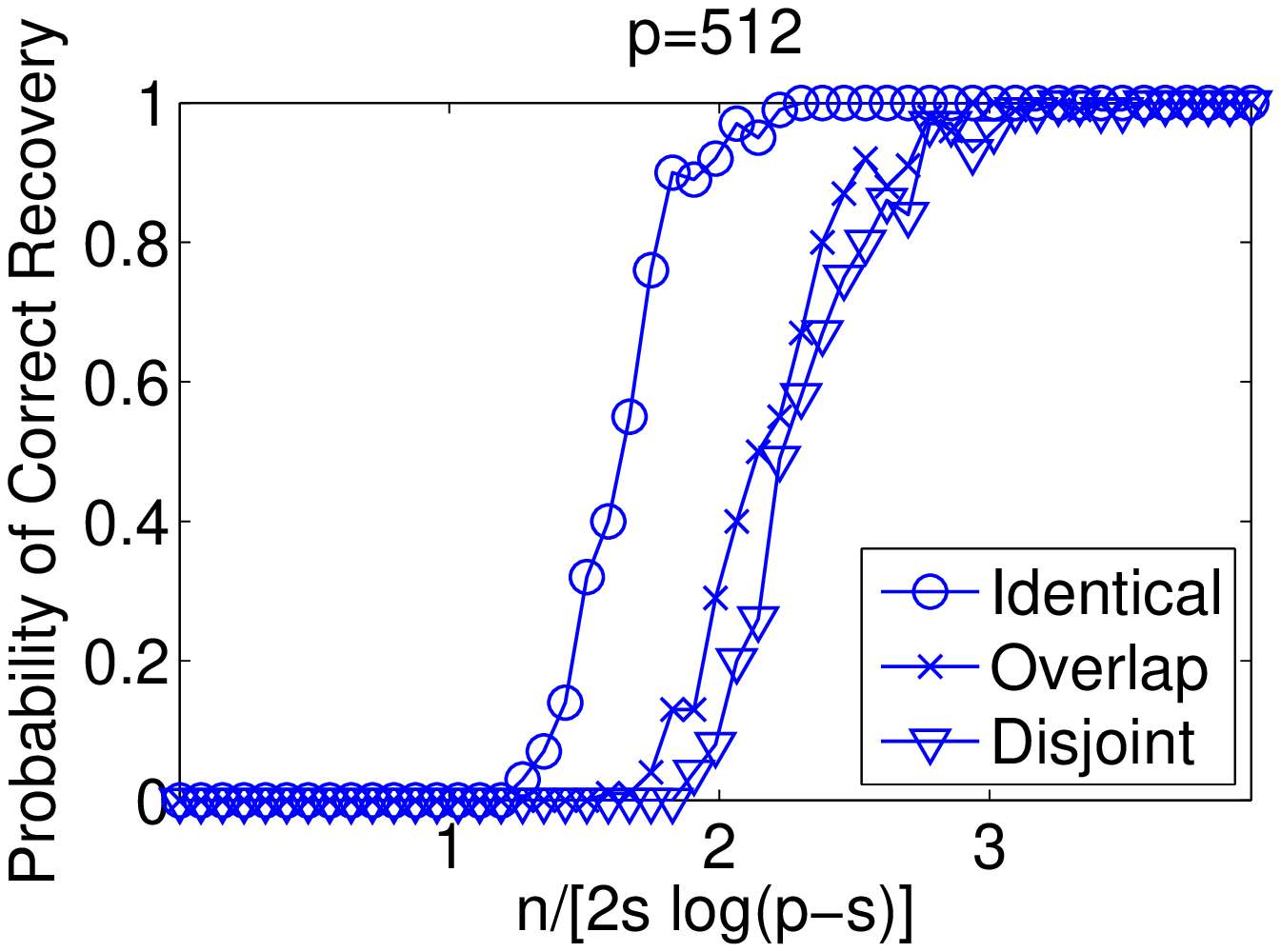}
\end{center}
\vspace{-0.2cm}
\caption{Impact of overlapping levels of support sets on the sample size with non-equal regression values for overlapping entries and identical covariance matrices across tasks}
\label{fig:overlapwpert}
\end{figure}

\begin{figure}[ht!]
\vspace{-0.2cm}
\begin{center}
\begin{tabular}{ccc}
\includegraphics[width=6.0cm]{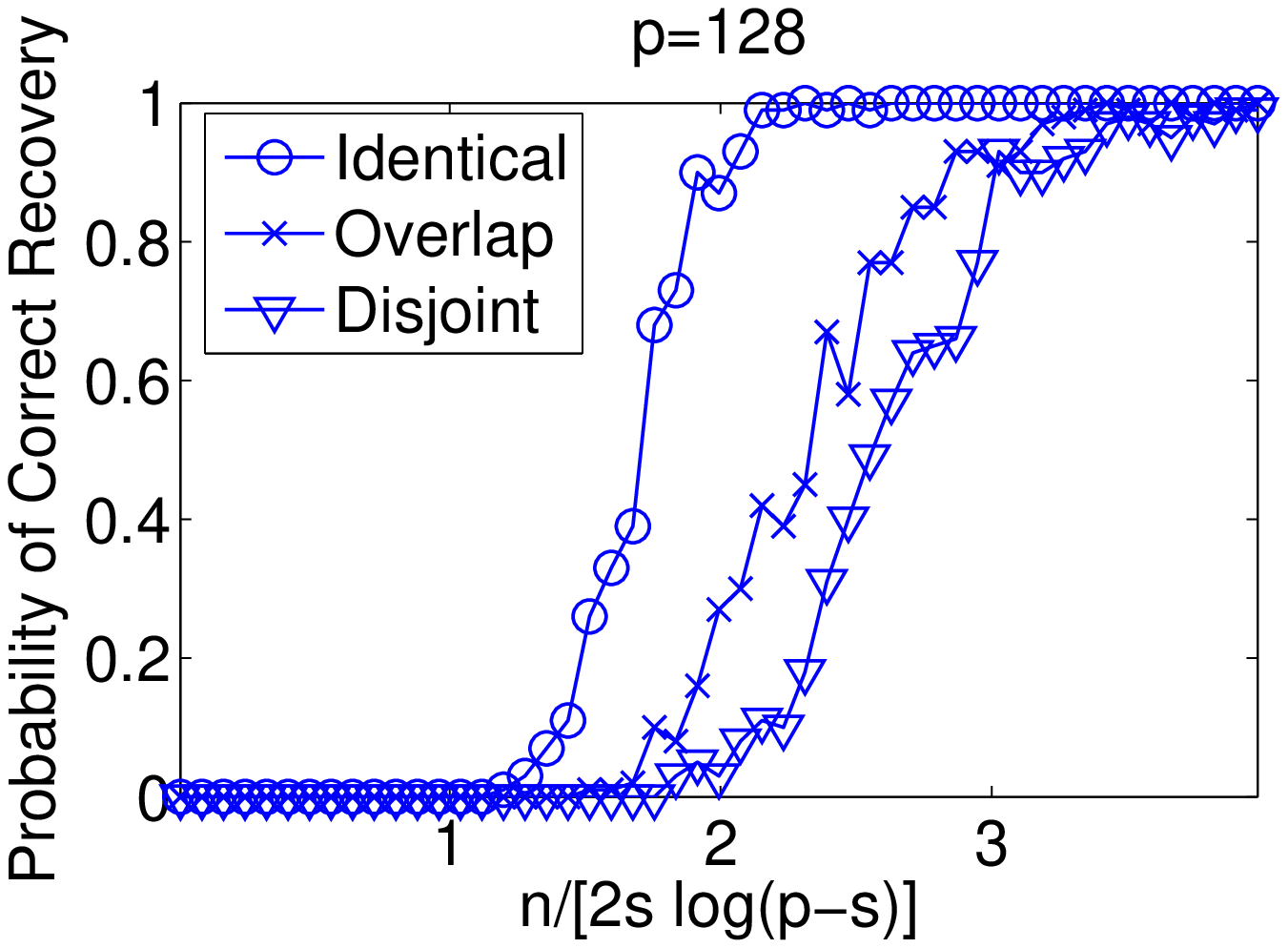} &
\includegraphics[width=6.0cm]{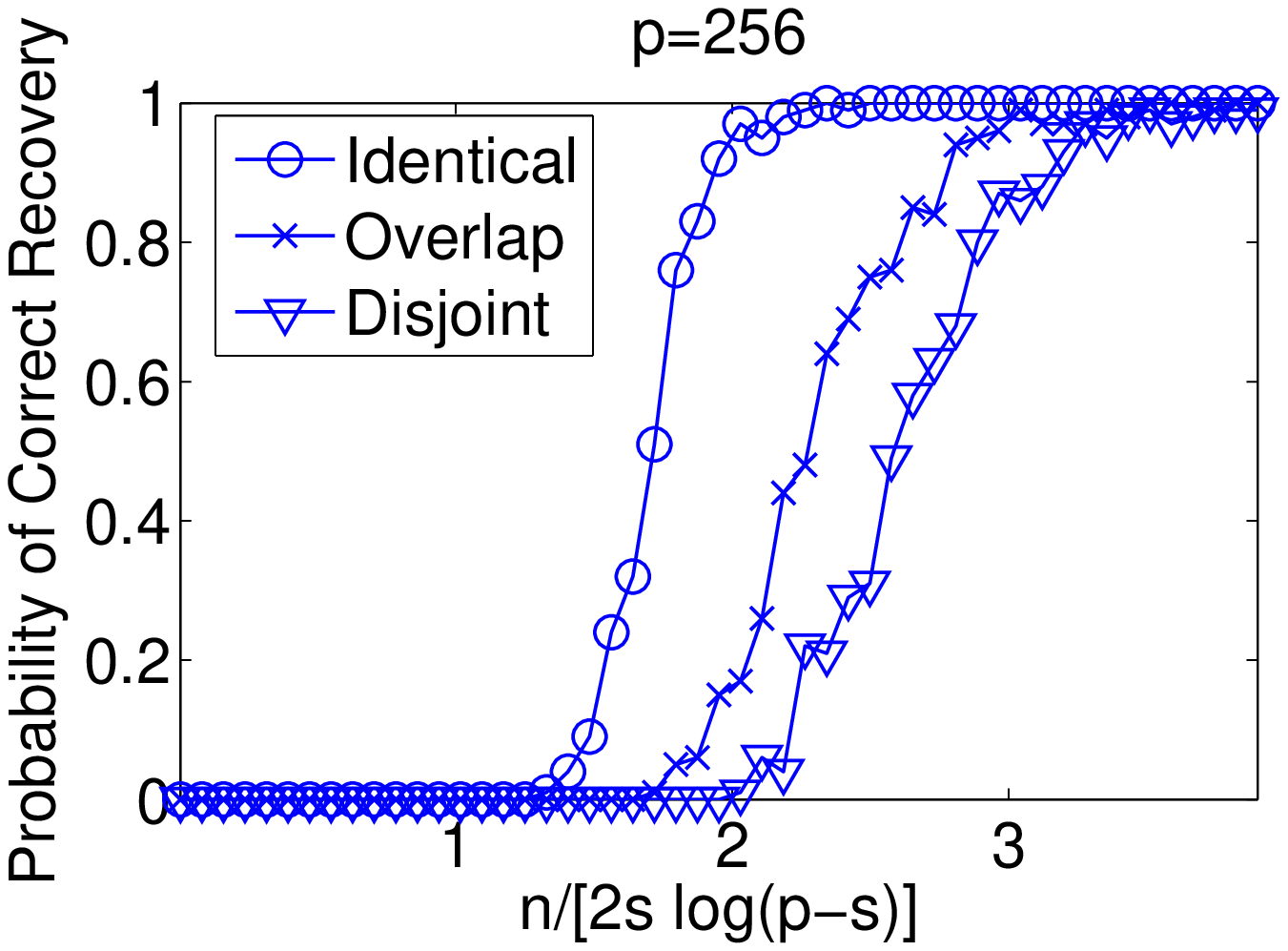}
\end{tabular}
\vspace{-0.2cm}
\includegraphics[width=6.0cm]{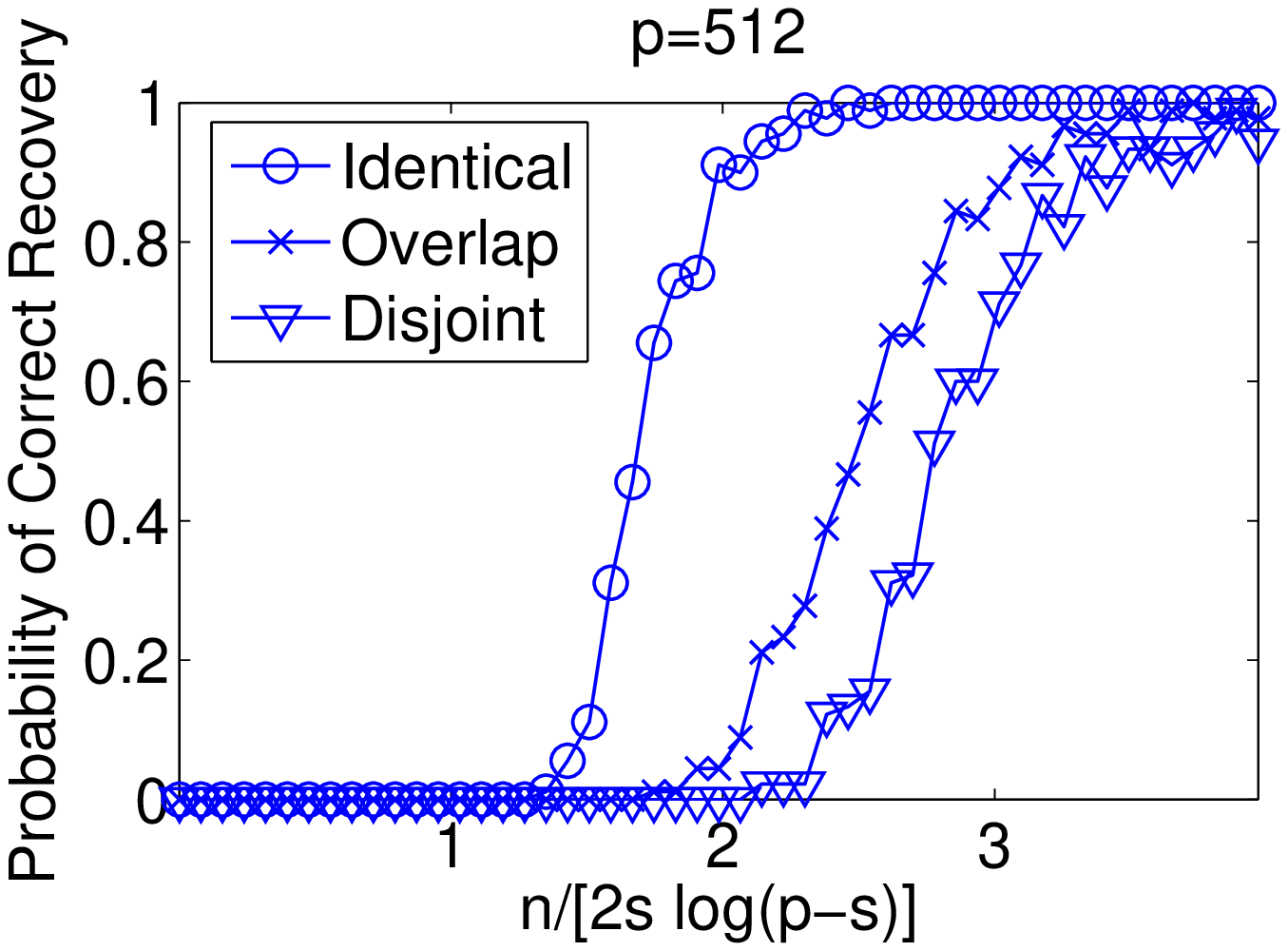}
\end{center}
\vspace{-0.2cm}
\caption{Impact of overlapping levels of support sets on the sample size with same regression values for overlapping entries and varying covariance matrices across tasks}
\label{fig:overlapsigma}
\end{figure}

The preceding experiment is taken for the case when the design matrices of the two tasks have the same covariance matrix and the regression vectors are identical on overlapping entries. It is interesting to investigate how non-equal values in regression vectors and different covariance matrices across the two tasks affect the sample complexity. We first study the case when the regression vectors of the two tasks do not have the same values on the overlapping entries. For the case when the two tasks have the same support sets, we let $\ubeta^{*(k)}_j=\frac{1}{\sqrt{K}}\times \left( 1+\frac{k}{16} \right)$ for $j=16t_{pe}$, and $\ubeta^{*(k)}_j=\frac{1}{\sqrt{K}}\times \left( 1-\frac{k}{16} \right)$  for $j=16t_{pe}+8$, where integer $t_{pe}\geq 0$ such that $j\leq p$ for $k=1,2$. For the overlapping model, $S_1$ and $S_2$ are the same as the preceding experiment. For $k=1$, $\ubeta^{*(k)}_j=1$ if $j=24t_{pe}+1$, and $\ubeta^{*(k)}_j=\frac{1}{\sqrt{K}}\times \left( 1+\frac{1}{16} \right)$ if $j=24t_{pe}+2$, where integer $t_{pe}\geq 0$ such that $j\leq p$. For $k=2$, $\ubeta^{*(k)}_j=\frac{1}{\sqrt{K}}\times \left( 1-\frac{1}{16} \right)$ if $j=24t_{pe}+2$, and $\ubeta^{*(k)}_j=1$ if $j=24t_{pe}+3$, where integer $t_{pe}\geq 0$ such that $j\leq p$. For the disjoint case, the regression vectors are the same as the preceding experiment since no overlapping exists in the disjoint model. Other parameters ($\Sigma^{(1:K)}, n, p, s, \lambda$) are kept the same as the preceding experiment. Fig.~\ref{fig:overlapwpert} plots the probability of correct recovery of the support union versus the scaled sample size for this experiment. It can be observed that Fig.~\ref{fig:overlapwpert} exhibits same behavior as Fig.~\ref{fig:overlap} and demonstrates that higher overlapping level across two tasks leads to smaller sample size needed for recovery, although the regression vectors do not match values for the overlapping entries. We also denote that more careful comparison of Fig.~\ref{fig:overlapwpert} and Fig.~\ref{fig:overlap} suggests that the model with perturbation on overlapping entries in regression vectors requires a slightly larger sample size than the model without perturbation.

We finally study how the varying covariance matrices across the two tasks influence the result. We set the covariance matrices $\Sigma^{(k)}$ for $k=1,2$ as follows. We let $Cov(X_a,X_b)>0$ ($a,b \in \{1,2,\ldots,p\}$) if $a=b\pm 1$, and otherwise $Cov(X_a,X_b)=0$. More specifically, we let $Cov(X_a,X_b)=1+1/k$ if $a=b\pm 1$ and $a$ is odd, and $Cov(X_a,X_b)=1-0.8/k$ if $a=b\pm 1$ and $a$ is even. Other parameters ($B^*, n, p, s, \lambda$) are the same as the experiment in Fig.~\ref{fig:overlap}. Fig.~\ref{fig:overlapsigma} compares the probability of correct recovery versus the scaled sample size for the three overlapping models under the varying covariance matrices but the same values for overlapping regression entries across the two tasks. The behavior is similar to that in Fig.~\ref{fig:overlap} and Fig.~\ref{fig:overlapwpert}. More careful comparison of Fig.~\ref{fig:overlapsigma} and Fig.~\ref{fig:overlap} suggests that the varying covariance matrices across the two tasks require larger sample size than the case with identical covariance matrices.

%The experiment 5 (whose result is displayed in Fig. ~\ref{fig:overlapsigma}) is the same with the original overlapping experiment (whose result is in ~\ref{fig:overlap}) except that design matrices are from different distributions instead of identical ones. It is noted the varying covariance matrices also make the sample size slightly larger, but not large enough to be taken as in order sense.

\section{Proof of Theorem \ref{th:achieve}}\label{sec:proofachieve}

Our proof applies the framework developed by  \cite{Wain09} and by  \cite{Oboz11} based on the idea of primal-dual witness. However, for the MVMR model, we need to develop novel adaption due to varying design matrices across tasks. In \cite{Oboz11}, since the model can be expressed by a
matrix operation on regression matrix, the proof involves many operations for matrices, for which properties/bounds for matrices can be applied. However, the MVMR model is expressed by $K$ operations on individual regression vectors. The proof mostly involves first manipulating/bounding individual regression vectors and then integrating these manipulations/bounds together for conditions across all tasks. Our adaption needs to make bounds in both steps as tight as possible in order to develop sharp threshold conditions. We next present our proof in detail.

The objective function in the multi-task Lasso problem given in \eqref{eq:optprob} is convex, and hence the following Karush-Kuhn-Tucker (KKT) condition is sufficient and necessary to characterize an optimal solution:
\begin{equation}\label{eq:kktcond}
\nabla_{B} f(B) + \lambda_n Z=0
\end{equation}
where $f(B)=\frac{1}{2n}\sum_{k=1}^{K} \left\|\uY^{(k)}-X^{(k)} \ubeta^{(k)}\right\|^2$, and $Z \in \partial \|B\|_{l_1 / l_2}$.

Before introducing the sufficient conditions, we first present the following lemma which provides an important property about the optimal solution to the above problem.
\begin{lemma}\label{lemma:zeroset}
Suppose there exists an optimal solution $\hB$ to the multi-task Lasso problem given in \eqref{eq:optprob}. Suppose $\hZ$ is in the subdifferential of $\|B\|_{l_1 / l_2}$ at $\hB$, and satisfies the KKT condition in \eqref{eq:kktcond} jointly with $\hB$. Suppose that $\hZ$ satisfies $\left\|\hZ_{\Omega}\right\|_{l_\infty/l_2} < 1$, where $\hZ_{\Omega}$ denotes the submatrix that contains rows of $\hZ$ with indices in the set $\Omega$. Then any optimal solution $\tB$ to \eqref{eq:optprob} must satisfy $\tB_{\Omega}=0$.
\end{lemma}
The proof of Lemma \ref{lemma:zeroset} is similar to that of Lemma 1 by \cite{Wain09}. For completeness of our paper, we provide the proof of Lemma \ref{lemma:zeroset} in Appendix $\ref{app:solution}$.

We now construct a pair $(\hB,\hZ)$ that satisfy the KKT condition in \eqref{eq:kktcond}. We first let $\hB_S$ be an optimal solution to the following optimization problem:
\begin{equation}\label{eq:bs}
\hB_S=\text{argmin}_{B_S} \left[ \left. f(B)\right|_{B_{S^c}=0}+\lambda_n \left\|B_S\right\|_{l_1 / l_2} \right]
\end{equation}
and let $\hZ_S$ be the associated element in the subdifferential of $\left\|B_S\right\|_{l_1 / l_2}$ such that $(\hB_S,\hZ_S)$ satisfy the KKT condition for the optimization problem given in \eqref{eq:bs}. We then let $\hat{B}_{S^c}=0$, and let $\hZ_{S^c}$ be an element in the subdifferential of $\left\|B_{S^c}\right\|_{l_1 / l_2}$ that satisfies the KKT condition jointly with $\hB_{S^c}=0$ for the following problem
\begin{equation}\label{eq:probsc}
\text{argmin}_{B_{S^c}} \left[\left. f(B)\right|_{B_{S}=\hB_S}+\lambda_n \left\|B_{S^c}\right\|_{l_1 / l_2} \right].
\end{equation}
Such $\hZ_{S^c}$ must exist if the KKT condition for the optimization problem \eqref{eq:probsc} implies $\left\|\hZ_{S^c}\right\|_{l_\infty/l_2} \leq 1$. Now it is easy to see that $(\hB,\hZ)$ obtained above satisfies the KKT condition in \eqref{eq:kktcond} and is hence an optimal solution to the problem \eqref{eq:optprob}. Furthermore, following Lemma \ref{lemma:zeroset}, if $\left\|\hZ_{S^c}\right\|_{l_\infty/l_2} < 1$, then any optimal solution $\tB$ to \eqref{eq:optprob} satisfies $\tB_{S^c}=0$. Therefore, condition $\left\|\hZ_{S^c}\right\|_{l_\infty/l_2} < 1$ guarantees both that there exists an optimal solution with the structure described as above and that all optimal solutions $\tB$ satisfies $\tB_{S^c}=0$. Furthermore, the condition $\left\|\hZ_{S^c}\right\|_{l_\infty/l_2} < 1$ guarantees uniqueness of the optimal solution. The arguments follow from the proof of Lemma 2 by \cite{Wain09}.

We next proceed to characterize the conditions that guarantee $\left\|\hZ_{S^c}\right\|_{l_\infty/l_2} < 1$. For $j\in S^c$ and $k=1,\ldots,K$, we have
\begin{flalign}
\hZ_{jk}=&-\frac{1}{\lambda_nn}{\uX_j^{\pk}}^T \left(\Pi_S^{\pk}-I_n \right)\uW^{\pk}+\frac{1}{n} {\uX_j^{\pk}}^TX_S^{\pk}\left(\hSigma_{SS}^{\pk}\right)^{-1}\huZ_{Sk},
\end{flalign}
where $\uX_j^{(k)}$ denotes the $j$th column of the matrix $X^{(k)}$, $\hSigma_{SS}^{\pk}=\frac{1}{n} {X_S^{\pk}}^TX_S^{\pk}$, and $\Pi_S^{\pk}=\frac{ X_S^{\pk}\left(\hSigma_{SS}^{\pk}\right)^{-1}{X_S^{\pk}}^T}{n}$. The steps to obtain the above $ \hZ_{jk}$ is provided in Appendix $\ref{app:hZSc}$ for completeness.

{\em Analysis of $V_{S^c}$:} We let $V_j =\left(\hZ_{j1},\ldots,\hZ_{jK}\right)$. We need to characterize the conditions so that $\left\|V_j\right\|_{l_2} < 1$ for all $j \in S^c$ with high probability. We write $V_j$ into three terms as follows
\begin{flalign}\label{eq:vj}
V_j&=\underbrace{\mE\left(V_j\mid X_S^{(1:K)}\right)}_{T_{j1}}  \nn \\
&+\underbrace{\mE\left(V_j\mid X_S^{(1:K)},\uW^{(1:K)}\right)-\mE\left(V_j\mid X_S^{(1:K)}\right)}_{T_{j2}} \nn \\
&+\underbrace{V_j-\mE\left(V_j\mid X_S^{(1:K)},\uW^{(1:K)}\right)}_{T_{j3}}
\end{flalign}
where $X_S^{(1:K)}=\left(X_S^{(1)},\ldots,X_S^{(K)}\right)$ and $\uW^{(1:K)}=\left(\uW^{(1)},\ldots,\uW^{(K)}\right)$. %(\begin{color}{red} change $\uW$ to $\uW^{(1:K)}$ when needed \end{color}).
We next evaluate $T_{j1},T_{j2}$, and $T_{j3}$ one by one.

{\em Evaluation of $T_{j1}$:} By the definition of $\hZ_S$, we have the following conditional independencies:
\begin{flalign}
\left(\uW^{(k)}\perp\uX_j^{(k)}\Big|X_S^{(1:K)}\right), \quad \left(\huZ_{Sk}\perp\uX_j^{(k)}\Big|X_S^{(1:K)}\right), \left(\huZ_{Sk}\perp\uX_j^{(k)}\Big|X_S^{(1:K)}, \uW^{(1:K)}\right).
\end{flalign}

Given the above independence properties, we first derive
\begin{flalign}
\mE\left(\hZ_{jk}\Big|X_S^{(1:K)}\right)=&-\frac{1}{\lambda_nn}\mE\left({\uX_j^{\pk}}^T\Big|X_S^{(1:K)}\right) \left(\Pi_S^{\pk}-I_n \right)\mE\left(\uW^{\pk}\right)  \nn \\
&+\frac{1}{n}\mE\left({\uX_j^{\pk}}^T\Big|X_S^{(1:K)}\right)X_S^{\pk}\left(\hSigma_{SS}^{\pk}\right)^{-1}\mE\left(\huZ_{Sk}\Big|X_S^{(1:K)}\right) \nn \\
%=& \frac{1}{n}\Sigma_{jS}^{\pk} \left(\Sigma_{SS}^{\pk} \right)^{-1}{ X_S^{\pk}}^T X_S^{\pk}\left(\hSigma_{SS}^{\pk}\right)^{-1}\mE\left(\huZ_{Sk}|X_S^{(1:K)}\right) \nn \\
=& \Sigma_{jS}^{\pk} \left(\Sigma_{SS}^{\pk} \right)^{-1}\mE\left(\huZ_{Sk}\Big|X_S^{(1:K)}\right)
\end{flalign}
for $j\in{S^c}$, where $\Sigma_{jS}^{\pk}$ represents the covariance between a component in $\uX_j^{(k)}$ and a row in $X_S^{(k)}$. We then obtain the following bound on $\|T_{j1}\|_{l_2}$ with the proof provided in Appendix \ref{app:t1norm}:
\begin{flalign}
\|T_{j1}\|_{l_2} \leq \sum_{a=1}^{|S|} A_{ja} ,
\end{flalign}
where $A_{ja}=\max_k\left|\left(\Sigma_{S^cS}^{\pk} \left(\Sigma_{SS}^{\pk} \right)^{-1}\right)_{ja}\right|$ for $j\in S^c$ and $a \in S$. We hence obtain
\[\max_{j\in S^c} \|T_{j1}\|_{l_2} \leq \max_{j\in S^c}\sum_{a=1}^{|S|} A_{ja} = \vertiii{A}_\infty \leq 1-\gamma .\]

{\em Evaluation of $T_{j2}$:} Due to the independency $\left(\huZ_{Sk}\perp\uX_j^{(k)}|X_S^{(1:K)}, \uW^{(1:K)}\right)$, we obtain
\begin{flalign}
\mE&\left(\hZ_{jk}  \Big|X_S^{(1:K)},\uW^{(1:K)}\right)  \nn \\
=& -\frac{1}{\lambda_nn}\mE\left({X_j^{\pk}}^T\Big|X_S^{(1:K)},\uW^{(1:K)}\right) \left(\Pi_S^{\pk}-I_n \right)\uW^{\pk}  \nn \\
&+\frac{1}{n}\mE\left({X_j^{\pk}}^T\Big|X_S^{(1:K)},\uW^{(1:K)}\right) X_S^{\pk}\left(\hSigma_{SS}^{\pk}\right)^{-1}\mE\left(\huZ_{Sk}\Big|X_S^{(1:K)},\uW^{(1:K)}\right)  \nn \\
=& \Sigma_{jS}^{\pk} \left(\Sigma_{SS}^{\pk} \right)^{-1}\huZ_{Sk}
\end{flalign}
where the second equality follows because $\huZ_{Sk}$ is a function of $X_S^{(1:K)}$ and $\uW^{(1:K)}$. We then obtain
\begin{flalign}
\mE&\left(\hZ_{jk}\Big|X_S^{(1:K)},\uW^{(1:K)}\right)-\mE\left(\hZ_{jk}\Big|X_S^{(1:K)}\right)=\Sigma_{jS}^{\pk} \left(\Sigma_{SS}^{\pk} \right)^{-1}\left( \huZ_{Sk}-\mE\left(\huZ_{Sk}\Big|X_S^{(1:K)}\right) \right).
\end{flalign}
Thus, following from steps similar to those in Appendix \ref{app:t1norm}, we obtain
\begin{flalign}
\|T_{j2}\|_{l_2} & \leq \sum_{a=1}^{|S|} A_{ja} \left\| \hZ_{S}-\mE(\hZ_{S}|X_S^{(1:K)}) \right\|_{l_\infty/l_2},
\end{flalign}
and hence
\begin{flalign} \label{eq:tj2boundl2}
\max_{j\in S^c}\|T_{j2}\|_{l_2}  \leq & \max_{j\in S^c}\sum_{a=1}^{|S|} A_{ja} \left\| \hZ_{S}-\mE(\hZ_{S}|X_S^{(1:K)}) \right\|_{l_\infty/l_2} \nn \\
 = &\vertiii{A}_\infty \left\| \hZ_{S}-\mE(\hZ_{S}|X_S^{(1:K)}) \right\|_{l_\infty/l_2} \nn \\
 \leq & (1-\gamma) \left\| \hZ_{S}-\mE(\hZ_{S}|X_S^{(1:K)})\right\|_{l_\infty/l_2} \nn \\
%& \leq (1-\gamma) \left(|| \hZ_{S}-Z_S^*||_{l_\infty/l_2} +|| \mE(\hZ_{S}-Z_S^*|X_S^{(1:K)})||_{l_\infty/l_2} \right) \nn \\
 \leq &(1-\gamma) \left\| \hZ_{S}-Z_S^*\right\|_{l_\infty/l_2}+(1-\gamma)\mE\left[ \left\| \hZ_{S}-Z_S^*\right\|_{l_\infty/l_2} \Big|X_S^{(1:K)} \right].
\end{flalign}

We next provide the following lemma given by \cite{Oboz11}, which is useful for our proof.
\begin{lemma}(\cite{Oboz11})\label{lemma:delta}
Consider the matrix $\Delta \in R^{S\times K}$ with rows $\Delta_i :=\frac{\hB_i-B_i^*}{\|B_i^*\|_2}$. If $\|\Delta\|_{l_{\infty}/l_2} < \frac{1}{2}$, then
\[ \|\hZ_S-Z_S^*\|_{l_\infty/l_2}\leq 4\|\Delta\|_{l_\infty/l_2}.\]
\end{lemma}

By applying the above lemma, given the condition $\|\Delta\|_{l_{\infty}/l_2}< \frac{1}{2}$ that we will show later, we obtain
\begin{flalign}
\max_{j\in S^c}&\|T_{j2}\|_{l_2} \leq  4(1-\gamma) \left(\|\Delta\|_{l_\infty/l_2} +\mE\left[ \|\Delta\|_{l_\infty/l_2} \Big|X_S^{(1:K)} \right]\right)\nn
\end{flalign}

We will show later in the analysis of $U_S$ that $\|\Delta\|_{l_\infty/l_2}$ is of order $o(1)$ with high probability, and hence the above inequality holds with high probability.

{\em Evaluation of $T_{j3}$:} We introduce the vector $\uD^{(k)}$ such that
\begin{flalign}
\hZ_{jk}=&-\frac{1}{\lambda_nn}{\uX_j^{\pk}}^T \left(\Pi_S^{\pk}-I_n \right)\uW^{\pk}+\frac{1}{n} {\uX_j^{\pk}}^TX_S^{\pk}\left(\hSigma_{SS}^{\pk}\right)^{-1}\huZ_{Sk}   \nn \\
=&{\uX_j^{\pk}}^T\uD^{\pk}.
\end{flalign}
It is clear that for $j\in S^c$,
\[Cov\left(\uX_j^{\pk}|X_S^{(1:K)},\uW^{(1:K)}\right)=\left(\Sigma_{S^cS^c|S}^{\pk} \right)_{jj}I_n.\]
Under the condition that $X_S^{(1:K)}$ and $\uW^{(1:K)}$ are given, we have
\begin{flalign}\label{eq:zjkdist}
\left(\hZ_{jk}|X_S^{(1:K)},\uW^{(1:K)}\right)-\mE\left[\hZ_{jk}|X_S^{(1:K)},\uW^{(1:K)}\right] \sim \cN(0,\sigma_{jk}^2)
\end{flalign}
where
\begin{flalign}
\sigma_{jk}^2=&\frac{1}{n}\left(\Sigma_{S^cS^c|S}^{\pk} \right)_{jj} {\huZ_{Sk}}^T\left(\hSigma_{SS}^{\pk}\right)^{-1}\huZ_{Sk}-\frac{1}{n^2\lambda_n^2}\left(\Sigma_{S^cS^c|S}^{\pk} \right)_{jj} {\uW^{\pk}}^T\left(\Pi_S^{\pk}-I_n \right)\uW^{\pk} .
\end{flalign}
%The computation steps are in appendix $\ref{app:sigmajk}$.

Given $\left(X_S^{(1:K)},\uW^{(1:K)}\right)$, $\hZ_{jk}$ is independently distributed across $k$ for $k=1,\ldots,K$. Hence,
\begin{flalign}
 \hspace{0mm}\hZ_{jk}-\mE\left[\hZ_{jk}|X_S^{(1:K)},\uW^{(1:K)}\right]\overset{d.}{=}\sigma_{jk} \xi_{jk}\quad \text{given} \quad \left(X_S^{(1:K)},\uW^{(1:K)}\right)
 \end{flalign}
where $\xi_{jk} \sim \cN(0,1)$ is independently distributed across $k$ for $k=1,\ldots,K$. Thus,
\begin{flalign}
\left\|V_j-\mE\left[V_j\Big|X_S^{(1:K)},\uW^{(1:K)}\right]\right\|^2_{l_2}\overset{d.}{=}\sum_{k=1}^K \sigma_{jk}^2\xi_{jk}^2\quad \text{given} \quad\left(X_S^{(1:K)},\uW^{(1:K)}\right).
\end{flalign}

We hence obtain
\begin{flalign}\label{eq:t3}
\max_{j\in S^c} \left\| T_{j3}\right\|_{l_2}^2 & \overset{d.}{=}\max_{j\in S^c} \sum_{k=1}^K \sigma_{jk}^2\xi_{jk}^2  \nn \\
& \leq \max_{j\in S^c}\max_{1\leq k\leq K}\sigma_{jk}^2 \max_{j\in S^c}\left(\sum_{k=1}^K \xi_{jk}^2\right)\quad \text{given} \quad\left(X_S^{(1:K)}, \uW^{(1:K)}\right)
\end{flalign}

We next provide a useful bound for $\chi^2$ random variable, which was given by \cite{Oboz11}.
\begin{lemma}(\cite{Oboz11})\label{lemma:chi2}
Let $Z$ be a central $\chi^2$ distributed random variable with the degree $d$. Then for all $t > d$, we have
\[P(Z \ge 2 t)\leq \exp\left(-t\left[1-2\sqrt{\frac{d}{t}}\right]\right).\]
\end{lemma}
Applying the above lemma, we obtain for all $t > K$,
\begin{flalign}\label{eq:lm3bound}
P\left(\max_{j\in S^c}\left(\sum_{k=1}^K \xi_{jk}^2\right) > 2t\right)&\leq (p-s)P\left(\left(\sum_{k=1}^K \xi_{jk}^2\right) > 2t \right) \nn \\
& \leq (p-s) \exp\left(-t\left[1-2\sqrt{\frac{K}{t}}\right]\right)
\end{flalign}

By applying the bound on $\sigma^2_{jk}$ derived in appendix $\ref{app:boundsigma}$ together with \eqref{eq:lm3bound}, we further have
\begin{flalign}\label{eq:t3bound}
&\max_{j\in S^c} \|T_{j3}\|_{l_2}^2 \leq 2t\rho_u\left(\Sigma^{(1:K)}_{S^cS^c|S}\right)\left( \frac{\psi(B^*,\Sigma^{(1:K)})}{n}+\Gamma \right)
\end{flalign}
with the probability larger than
\begin{flalign}\label{eq:t3prob}
 1&-2(K+1)\exp\left(-\frac{s}{2}\right)-4(K+1)\exp\left(-\frac{n}{2}\left(\frac{1}{4}-\sqrt{\frac{s}{n}}\right)_+^2\right) \nn \\
&-K\exp\left( -\log{s}+2\sqrt{2\log s} \right)-(p-s)\exp\left( -t\left[ 1-2\sqrt{\frac{K}{t}} \right] \right)  \nn \\
&-\exp{\left( -5(n-s)\left[ 1-2\sqrt{\frac{1}{5}} \right] \right)}
\end{flalign}
for $t>K$, where
\begin{flalign}\label{eq:gamma}
\Gamma=\frac{16s\left\|\Delta\right\|_{l_\infty/l_2}}{nC_{min}}(1+2\left\|\Delta\right\|_{l_\infty/l_2})+\frac{12}{C_{min}}{\left(\frac{s}{n}\right)}^{\frac{3}{2}}+\frac{10(n-s){\sigma_W^{(k)}}^2}{n^2\lambda_n^2}.
\end{flalign}
For $n$ large enough, $\Gamma$ converges to zero with an order $o\left( \frac{s}{n} \right)$. We also note that $\psi(B^*,\Sigma^{(1:K)})$ has an order $O(s)$ based on Proposition $\ref{pro:psiorder}$. In \eqref{eq:t3bound}, we set $t=\frac{1+v}{1+\delta}\log{(p-s)}$ where $v>0$ and $\delta=v/(3v+4)$. We can then show that if
\[n>2(1+v) \psi\left(B^*,\Sigma^{(1:K)}\right)\log(p-s)\frac{\rho_u\left(\Sigma^{(1:K)}_{S^cS^c|S}\right)}{\gamma^2},\]
then
\begin{flalign}
\max_{j\in S^c}\| T_{j3} \|_{l_2}< \gamma
\end{flalign}
with the probability larger than
\begin{flalign}
 1&-2(K+1)\exp\left(-\frac{s}{2}\right)-4(K+1)\exp\left(-\frac{n}{2}\left(\frac{1}{4}-\sqrt{\frac{s}{n}}\right)_+^2\right) \nn \\
&-K\exp\left(-\log{s}+2\sqrt{2\log s} \right)-\exp\left( -\frac{v}{2}\log{(p-s)} \right) \nn \\
&-\exp{\left( -5(n-s)\left[ 1-2\sqrt{\frac{1}{5}} \right] \right)}.
\end{flalign}

It follows from \eqref{eq:vj} that
\[ \left\|V_j\right\|_{l_2} \leq \left\|V_{j1}\right\|_{l_2}+\left\|V_{j2}\right\|_{l_2}+\left\|V_{j3}\right\|_{l_2}. \]
Combining the above equation with the evaluation for $T_{j1}$, $T_{j2}$, $T_{j3}$, we conclude that $\left\|V_j\right\|_{l_2}<1$.

{\em Analysis of $U_S$:} We have obtained the sufficient conditions for the existence and uniqueness of an optimal solution to the problem given in \eqref{eq:optprob}, which guarantees $\hB_{S^c}=0$. It remains to characterize conditions such that all rows of $\hB_S$ are nonzero and hence $S(\hB)$ recovers the true support union.

In order to guarantee that every row of $\hB_S$ is nonzero, it suffices to guarantee that
\[ \left\|U_{S}\right\|_{l_\infty/l_2} \leq \frac{1}{2}b_{min}^* \]
where
\[ U_S=\hB_S-B_S^*=\left[\uU_S^{(1)} \dots \uU_S^{(K)} \right].\]
Each column $\uU_S^{(k)}$ is given by
\begin{flalign}
 \uU_S^{(k)}&:=\hubeta_S^{(k)}-\ubeta_S^{*(k)}={\left(\hSigma_{SS}^{(k)}\right)}^{-1}\left( \frac{1}{n}X_S^{(k)T}\uW^{(k)}-\lambda_n\huZ_{Sk} \right).
\end{flalign}

It suffices to guarantee that
\[ \left\|\uU_S^{(k)}\right\|_{l_\infty} \leq \frac{1}{2K}b_{min}^*, \]
for $k=1,\ldots,K$. In order to bound $\left\|\uU_S^{(k)}\right\|_{l_\infty}$, we define
\[ \tuW^{(k)}=\frac{1}{\sqrt{n}}\left(\hSigma_{SS}^{(k)}\right)^{-\frac{1}{2}}X_S^{(k)T}\uW^{(k)}, \]
and hence
\[ \uU_S^{(k)}=\frac{1}{\sqrt{n}}\left(\hSigma_{SS}^{(k)}\right)^{-\frac{1}{2}}\tuW^{(k)}-\lambda_n\left(\hSigma_{SS}^{(k)}\right)^{-1}\huZ_{Sk}. \]
We then obtain the following bound
\begin{flalign}
\left\|\uU_S^{(k)}\right\|_{l_\infty} \leq& \underbrace{\left\|\frac{1}{\sqrt{n}}\left(\hSigma_{SS}^{(k)}\right)^{-\frac{1}{2}}\tuW^{(k)}\right\|_{l_\infty}}_{T_{k1}^{\prime}}+\underbrace{\lambda_n\left\|\left(\hSigma_{SS}^{(k)}\right)^{-1}\huZ_{Sk}\right\|_{l_\infty}}_{T_{k2}^{\prime}}.
\end{flalign}
We next evaluate the bounds on the two terms $T_{k2}^{\prime}$ and $T_{k1}^{\prime}$, respectively.

{\em Evaluation of $T_{k2}^{\prime}$:} We first derive the following bound
\begin{flalign}
\left\|\left(\hSigma_{SS}^{(k)}\right)^{-1}\huZ_{Sk}\right\|_{l_\infty}& \leq \max_{i\in S}\sum_{j \in S}\left|\left( \left(\hSigma_{SS}^{(k)}\right)^{-1} \right)_{ij}\right|  \nn \\
&= \vertiii{\left(\hSigma_{SS}^{(k)}\right)^{-1}}_{\infty} \nn \\
& \leq \vertiii{\left(\Sigma_{SS}^{(k)}\right)^{-1}}_{\infty}+\vertiii{\left(\hSigma_{SS}^{(k)}\right)^{-1}-\left(\Sigma_{SS}^{(k)}\right)^{-1}}_{\infty} \nn \\
& \overset{(a)}{\leq} D_{max}+\sqrt{s}\vertiii{\left(\hSigma_{SS}^{(k)}\right)^{-1}-\left(\Sigma_{SS}^{(k)}\right)^{-1}}_2 \nn \\
& \overset{(b)}{\leq} D_{max}+\frac{12s}{C_{min}\sqrt{n}}
\end{flalign}
with probability larger than $1-2\exp\left(-\frac{s}{2}\right)-3\exp\left(-\frac{n}{2}\left(\frac{1}{4}-\sqrt{\frac{s}{n}}\right)_+^2\right)$. In the above derivation, step (a) follows from the assumption of the theorem and $\vertiii{A}_{\infty}\leq \sqrt{s}\vertiii{A}_2$ for $A\in \mathbb{R}^{s\times n}$, and step (b) applies the bound given in \eqref{eq:xtximsi} in Appendix \ref{app:sigmass2}. Therefore,
\begin{flalign}
T_{k2}^{\prime} &\leq\lambda_n\left( D_{max}+\frac{12s}{C_{min}\sqrt{n}}\right)
\end{flalign}
with probability larger than $1-2\exp\left(-\frac{s}{2}\right)-3\exp\left(-\frac{n}{2}\left(\frac{1}{4}-\sqrt{\frac{s}{n}}\right)_+^2\right)$.

{\em Evaluation of $T_{k1}^{\prime}$:} We first have
\begin{flalign}
\mE\left( \tuW^{(k)}\tuW^{(k)T} \Big|X_S^{(1:K)}\right) & =\mE\Big( \left(\hSigma_{SS}^{(k)}\right)^{-\frac{1}{2}}\frac{1}{n} X_S^{(k)T}\uW^{(k)}\uW^{(k)T} X_S^{(k)}\left(\hSigma_{SS}^{(k)}\right)^{-\frac{1}{2}} \Big|X_S^{(1:K)} \Big) \nn \\
& ={\sigma_W^{(k)}}^2 I_S
\end{flalign}
which implies that given $X_S^{(1:K)}$, $\tuW^{(k)}$ has i.i.d.\ components with each being Gaussian distributed as $N\left(0, {\sigma_w^{(k)}}^2\right)$. Hence, given $X_S^{(1:K)}$, we have
\begin{flalign}
T_{k1}^{\prime}& \leq \vertiii{\left(\hSigma_{SS}^{(k)}\right)^{-\frac{1}{2}}}_{\infty} \left\|\frac{\tuW^{(k)}}{\sqrt{n}}\right\|_{l_\infty} \leq \sigma_{max}\sqrt{\frac{2s}{C_{min}}}\max_{j\in S}\sqrt{\frac{1}{n}\xi_j^2}
\end{flalign}
with probability larger than $1-\exp\left(-\frac{n}{2}\left(\frac{1}{4}-\sqrt{\frac{s}{n}}\right)_+^2\right)$, where $\sigma_{max}=\max_{1\leq k\leq K}\sigma_W^{(k)}$, and $\xi_j$ is the standard Gaussian random variable. The second inequality in the preceding derivation follows because $\vertiii{A}_{\infty}\leq \sqrt{s}\vertiii{A}_2$ for $A\in \mathbb{R}^{s\times n}$, and from the bound \eqref{eq:xtxi} provided in Appendix \ref{app:sigmass2}. By applying Lemma \ref{lemma:chi2} with $d=1$, we have
\begin{flalign}
P\left( \frac{1}{n}\max_{j\in S}\xi_j^2 \geq \frac{2t}{n}\right) \leq s\cdot \exp\left( -t\left[ 1-2\sqrt{\frac{1}{t}} \right] \right).
\end{flalign}
By setting $t=2 \log{s}$ in the above bound, we then obtain
\begin{flalign}
T_{k1}^{\prime} &\leq \sigma_{max}\sqrt{\frac{2s}{C_{min}}}\cdot\sqrt{\frac{2t}{n}} \leq \sqrt{\frac{8s\log{(s)}\sigma_{max}^2}{nC_{min}}}
\end{flalign}
with the probability larger than
\begin{flalign}
1&-\exp\left(-\frac{n}{2}\left(\frac{1}{4}-\sqrt{\frac{s}{n}}\right)_+^2\right)-\exp\left( -\log{s}+2\sqrt{2\log s} \right).
\end{flalign}

Combining the bounds on $T_{k1}^{\prime}$ and $T_{k2}^{\prime} $, we obtain
\begin{flalign}
\left\|\uU_S^{(k)}\right\|_{l_\infty}  &  \leq \sqrt{\frac{8s\log{(s)}\sigma_{max}^2}{nC_{min}}}+\lambda_n\left( D_{max}+\frac{12s}{C_{min}\sqrt{n}}\right)  \nn \\
& = \rho(n, s, \lambda_n)
\end{flalign}
with the probability larger than
\begin{flalign}
&1-2\exp\left(-\frac{s}{2}\right)-4\exp\left(-\frac{n}{2}\left(\frac{1}{4}-\sqrt{\frac{s}{n}}\right)_+^2\right)-\exp\left( -\log{s}+2\sqrt{2\log s} \right).
\end{flalign}
Thus, the assumption $\frac{\rho(n, s, \lambda_n)}{b_{min}^*}=o(1)$ guarantees that $\left\|\uU_S^{(k)}\right\|_{l_\infty} \leq \frac{1}{2K}b_{min}^*$ for sufficiently large $n$.

Furthermore, we derive the following bound
\begin{flalign}
\left\|\Delta\right\|_{l_\infty/\l_2}&\leq \frac{\left\|\hB_S-B^*_S\right\|_{l_\infty/\l_2}}{\min_{j\in S}\left\|B_j^*\right\|_2}
=\frac{\left\|U_S\right\|_{l_\infty/\l_2}}{b_{min}^*} \nn \\
&\leq \frac{\max_{j\in S}\sum_{k=1}^K|U_{jk}|}{b_{min}^*}
\leq \sum_{k=1}^K\frac{\max_{j\in S}|U_{jk}|}{b_{min}^*}  \nn \\
&= \sum_{k=1}^K\frac{\left\|\overrightarrow U_S^{(k)}\right\|_{l_\infty}}{b_{min}^*}
\leq \frac{K\rho(n, s, \lambda_n)}{b_{min}^*} =o(1)
\end{flalign}
with the probability larger than
\begin{flalign}
1&-2K\exp\left(-\frac{s}{2}\right)-4K\exp\left(-\frac{n}{2}\left(\frac{1}{4}-\sqrt{\frac{s}{n}}\right)_+^2\right)-K\exp\left( -\log{s}+2\sqrt{2\log s} \right).
\end{flalign}

Summarizing the analysis of $V_{S^c}$ and $U_S$, we conclude that the multi-task Lasso problem given in $\eqref{eq:optprob}$ has a unique solution $\hB$, whose support union recovers the true support union $S(B^*)$ with high probability under the assumption of the theorem.

\section{Proof of Theorem \ref{th:converse}}\label{sec:proofconverse}

Our proof follows the proof techniques established by  \cite{Oboz11} with further development due to varying design matrices across tasks.

Following from the proof in Section \ref{sec:proofachieve}, it can be shown that if either $\left\| \hZ_{S^c} \right\|_{l_\infty/l_2}>1$ holds or $\left\| \hB-B^* \right\|_{l_{\infty}/l_2}=o(b^*_{min})$ does not hold, no solution $\tilde{B}$ to the multi-task Lasso problem given in \eqref{eq:optprob} recovers the correct support union and satisfies $\left\| \tilde{B}-B^* \right\|_{l_{\infty}/l_2}=o(b^*_{min})$. Hence, if $\left\| \hB-B^* \right\|_{l_{\infty}/l_2}=o(b^*_{min})$ does not hold, it is already the case that the multi-task Lasso does not provide the desired solution. Then the following proof is to identify sufficient conditions such that $\left\| V_{S^c} \right\|_{l_\infty/l_2}>1$ when $\left\| \hB-B^* \right\|_{l_{\infty}/l_2}=o(b^*_{min})$ holds, where $V_j =\left( \hZ_{j1},\ldots,\hZ_{jK} \right)$ for $j\in S^c$.

%Remember $V_j =\left( \hZ_{j1},\ldots,\hZ_{jK} \right)$ for $j\in S^c$ and $\hZ_{S^c}$ is the subdifferential of $\|B_{S^c}\|_{l_1/l_2}$.
%Following from the proof and arguments in Appendix \ref{sec:proofachieve}, it is clear that if either $\left\| \hZ_{S^c} \right\|_{l_\infty/l_2}>1$ or $\left\| \hB-B^* \right\|_{l_{\infty}/l_2}=o(b^*_{min})$ is not satisfied, therefore no solution $\tilde{B}$ to the multi-task Lasso problem given in \eqref{eq:optprob} that recovers the correct support union and satisfy $\left\| \tilde{B}-B^* \right\|_{l_{\infty}/l_2}=o(b^*_{min})$. The following proof is to identify sufficient conditions such that $\left\| V_{S^c} \right\|_{l_\infty/l_2}>1$ and $\left\| \hB-B^* \right\|_{l_{\infty}/l_2}=o(b^*_{min})$ (i.e., one condition fails and one holds), where $V_j =\left( \hZ_{j1},\ldots,\hZ_{jK} \right)$ for $j\in S^c$. The following proof is to identify sufficient conditions such that $\left\| V_{S^c} \right\|_{l_\infty/l_2}>1$, where $V_j =\left( \hZ_{j1},\ldots,\hZ_{jK} \right)$ for $j\in S^c$.

%there does not exist $(B,Z)$ with $B_{S^c}=0$ and satisfying the KKT condition given in \eqref{eq:probsc}

%Following from the proof in Appendix \ref{sec:proofachieve}, it can be shown that if $\left\| \hZ_{S^c} \right\|_{l_\infty/l_2}>1$, then no solution $\tilde{B}$ to the multi-task Lasso problem given in \eqref{eq:optprob} recovers the correct support union and satisfy $\left\| \tilde{B}-B^* \right\|_{l_{\infty}/l_2}=o(b^*_{min})$.

We use the decomposition in \eqref{eq:vj}, which is rewritten below:
\[ V_j=T_{j1}+T_{j2}+T_{j3}. \]
However, we are now interested in lower bounding $\left\| V_{S^c} \right\|_{l_{\infty}/l_2}$. We first bound this quantity as follows:
\[ \left\| V_{S^c} \right\|_{l_{\infty}/l_2}\geq \left\| T_{S^c3} \right\|_{l_{\infty}/l_2}-\left\| T_{S^c1} \right\|_{l_{\infty}/l_2}-\left\| T_{S^c2} \right\|_{l_{\infty}/l_2}.\]
By the assumption of the theorem, $\left\| T_{S^c1} \right\|_{l_{\infty}/l_2}\leq 1-\gamma$. We next consider $T_{S^c 2}$. Due to \eqref{eq:tj2boundl2}, we have
\begin{flalign}
& \left\| T_{S^c2} \right\|_{l_{\infty}/l_2} \leq(1-\gamma) \left\| \hZ_{S}-Z_S^*\right\|_{l_\infty/l_2}  +(1-\gamma)\mE\left[ \left\| \hZ_{S}-Z_S^*\right\|_{l_\infty/l_2} \Big|X_S^{(1:K)} \right] .
 \end{flalign}
By the assumption that $\left\| \hB-B^* \right\|_{l_{\infty}/l_2}=o(b^*_{min})$ holds, following the proof in Section \ref{sec:proofachieve}, $\left\| T_{S^c2} \right\|_{l_{\infty}/l_2}=o(1)$ holds.

%with the probability
%\begin{flalign}
%1&-2K\exp\left(-\frac{s}{2}\right)-4K\exp\left(-\frac{n}{2}\left(\frac{1}{4}-\sqrt{\frac{s}{n}}\right)_+^2\right)\nn \\ &-K\exp\left( \log{s}-2\log{s}\left[ 1-2\sqrt{\frac{1}{2\log{s}}} \right] \right).
%\end{flalign}

%We still want to argue $\left\| T_{j2} \right\|_{l_{\infty}/l_2}=o(1)$, which is based on the condition that $\frac{\hB_i-B^*_i}{\|B_i^*\|_2}=o(1)$ for $i\in S$. Since we are only considering the existence of solution which is bounded with $\left\| \hB-B^* \right\|_{l_{\infty}/l_2}=o(b^*_{min})$, so

It then suffices to guarantee that $\left\| T_{S^c3} \right\|_{l_{\infty}/l_2}>2-\gamma$. We recall from \eqref{eq:t3} that
\begin{flalign}
&\max_{j\in S^c} \left\| T_{j3}\right\|_{l_2} \overset{d.}{=}\max_{j\in S^c} \sqrt{\sum_{k=1}^K \sigma_{jk}^2\xi_{jk}^2}  \quad \text{given} \quad\left( X_S^{(1:K)}, \uW^{(1:K)} \right),
\end{flalign}
where $\xi_{jk}\sim\cN(0,1)$ are independently distributed across $k$.

We let $V_{max}:=\left\| T_{S^c3} \right\|_{l_{\infty}/l_2}$, and the remaining part of the proof is to derive a lower bound on $V_{max}$, which takes several steps. The first step is to show that $V_{max}$ is concentrated around its expectation when $\left(X_S^{(1:K)}, \uW^{(1:K)}\right)$ are given.
\begin{lemma}\label{lemma:vmax}
For any $\delta>0$,
\begin{flalign}
& P\left[ |V_{max}-\mE V_{max}|\geq\delta \Big|X_S^{(1:K)}, \uW^{(1:K)} \right] \leq 4\exp{\left( -\frac{\delta^2}{2\rho_u\left(\Sigma^{(1:K)}_{S^cS^c|S}\right)\max_{1\leq k\leq K}{M_k}} \right)}.
 \end{flalign}
\end{lemma}
\begin{proof}
We first construct the following function $g: \mathbb{R}^{(p-s)\times K}\to\mathbb{R}$
\[ g(\xi):=\max_{j\in S^c} \left(\sqrt{\sum_{k=1}^K \sigma_{jk}^2\xi_{jk}^2}\right) \]
where $\xi_{jk}$ is the entry of the matrix $\xi$ with the index pair $\{ j,k \}$.

To explore the continuity property of the constructed function $g$, we let $u=(u_{jk}, j\in S^c,k=1,\ldots,K)$ and $v=(v_{jk}, j\in S^c,k=1,\ldots,K)$ be two matrices. We derive the following bound given $\left(X_S^{(1:K)}, \uW^{(1:K)}\right)$.
\begin{flalign}
\left| g(u)-g(v) \right| &= \left| \max_{j\in S^c}\left(\sqrt{\sum_{k=1}^K \sigma_{jk}^2 u_{jk}^2}\right) - \max_{n\in S^c}\left(\sqrt{\sum_{k=1}^K \sigma_{nk}^2 v_{nk}^2}\right)\right| \nn \\
&\leq \max_{j\in S^c}\left| \sqrt{\sum_{k=1}^K \sigma_{jk}^2 u_{jk}^2}-\sqrt{\sum_{k=1}^K \sigma_{jk}^2 v_{jk}^2} \right| \nn \\
%&\leq \max_{j\in S^c}\max_{1\leq k\leq K}\sigma_{jk} \max_{j\in S^c}\Big| \left\|u_j\right\|_2-\left\|v_j\right\|_2 \Big|  \nn \\
&\overset{(a)}{\leq} \left(\max_{j\in S^c}\max_{1\leq k\leq K}\sigma_{jk}\right) \left( \max_{j\in S^c} \left\|u_j-v_j\right\|_2\right)  \nn \\
&\leq \sqrt{\rho_u\left(\Sigma^{(1:K)}_{S^cS^c|S}\right)\max_{1\leq k\leq K}{M_k}} \left\|u-v\right\|_F,
\end{flalign}
where $(a)$ follows by taking square on both sides and comparing various cross terms.

Therefore, the function $g$ is Lipschitz continuous with constant $L=\sqrt{\rho_u\left(\Sigma^{(1:K)}_{S^cS^c|S}\right)\max_{1\leq k\leq K}{M_k}}$. The proof completes by applying Gaussian concentration inequality given below for a standard Gaussian vector $X$ and the Lipschitz function $g$ with the constant $L$:
\[ P\left[ |g(X)-\mE g(X)|\geq \delta \right] \leq 4\exp(-\delta^2/(2L^2)).\]
\end{proof}

The second step is to find a lower bound on $\mE [V_{max}]$.
\begin{lemma}
For any fixed $\delta'$ and sufficiently large $(p-s)$, the following inequality holds:
\begin{flalign}
&\mE\left[V_{max}\Big|X_S^{(1:K)}, \uW^{(1:K)}\right]\geq \max_{1\leq k\leq K}\sqrt{M_k}\sqrt{(1-\delta')\rho_l{\left(\Sigma^{(1:K)}_{S^cS^c|S}\right)}\log{(p-s)}/2}.\nn
\end{flalign}
\end{lemma}
\begin{proof}
The proof is under the assumption that $\left(X_S^{(1:K)}, \uW^{(1:K)}\right)$ are given. Define $\eta_{jk}=\sqrt{(\Sigma^{(k)}_{S^cS^c|S})_{jj}}\xi_{jk}$ and therefore, $\eta_{jk}\sim \cN\left(0,(\Sigma^{(k)}_{S^cS^c|S})_{jj}\right)$. We then have
\begin{flalign}
\sqrt{\sum_{k=1}^K \sigma_{jk}^2\xi_{jk}^2}=&\sqrt{\sum_{k=1}^K M_k (\Sigma^{(k)}_{S^cS^c|S})_{jj}\xi_{jk}^2}=\sqrt{\sum_{k=1}^K M_k \eta_{jk}^2}\geq  \sqrt{M_{k^*}} |\eta_{jk^*}|
\end{flalign}
where $k^*=\text{argmax}_{1\leq k\leq K}\sqrt{M_k}$. Without loss of generality, let $k^*=1$.
\begin{flalign}
\mE \left[ V_{max}\Big|X_S^{(1:K)}, \uW^{(1:K)} \right]\geq \sqrt{M_{k^*}}\cdot\mE\left( \max_{j\in S^c}|\eta_{j1}| \right)
\end{flalign}

The proof completes by applying the lower bound of $\mE\left( \max_{j\in S^c}|\eta_{j1}| \right)$. It can be shown that
\[ \mE \left[ (\eta_{i1}-\eta_{j1})^2 \right] \geq \rho_l{\left(\Sigma^{(1:K)}_{S^cS^c|S}\right)} \mE \left[ (\xi_{i1}-\xi_{j1})^2 \right]. \]
 By \cite{Sam93}, we have
\[ \mE\left( \max_{j\in S^c}|\eta_{j1}| \right)\geq \frac{1}{2}\sqrt{\rho_l{\left(\Sigma^{(1:K)}_{S^cS^c|S}\right)}}\mE \left( \max_{j\in S^c}|\xi_{j1}| \right) \]

Furthermore, the standard Gaussian random vector has the following bound by \cite{Ledo1991}:
\[ \mE \left( \max_{j\in S^c}|\xi_{j1}| \right) \geq \sqrt{2(1-\delta')\log{(p-s)}} \]
if $(p-s)$ is large enough, where $\delta'$ is a small positive number.
\end{proof}

In Appendix $\ref{app:boundsigma}$, we obtain the following lower bound
\[ \max_{1\leq k\leq K} M_k \geq \frac{\psi(B^*,\Sigma^{(1:K)})}{n}-\Gamma \]
with the probability larger than
\begin{flalign}
1&-2\exp\left(-\frac{s}{2}\right)-4\exp\left(-\frac{n}{2}\left(\frac{1}{4}-\sqrt{\frac{s}{n}}\right)_+^2\right)-\exp{\left( -5(n-s)\left[ 1-2\sqrt{\frac{1}{5}} \right] \right)}.
\end{flalign}

Since $\Gamma$ converges to $0$ with an order $o\left(\frac{s}{n}\right)$, $\max_{1\leq k\leq K} M_k\geq \frac{\psi(B^*,\Sigma^{(1:K)})}{n}(1-\delta'')$ holds for any small constant $\delta''>0$ and large enough $n$. We then have
\begin{flalign}
\mE  \left[V_{max}\Big|X_S^{(1:K)}, \uW^{(1:K)}\right] & \geq \sqrt{\frac{\psi(B^*,\Sigma^{(1:K)})}{n}(1-\delta'')} \sqrt{(1-\delta')\rho_l{\left(\Sigma^{(1:K)}_{S^cS^c|S}\right)}\log{(p-s)}/2} \nn \\
& \overset{(a)}{\geq} (2-\gamma)\sqrt{\frac{(1-\delta')(1-\delta'')}{4(1-v)}} \nn \\
& \overset{(b)}{>} 2-\gamma+\delta
\end{flalign}
with high probability, where (a) follows from the assumption of the theorem on the sample size $n$, and (b) follows by choosing $v>1-\frac{(1-\delta')(1-\delta'')}{4[1+(\delta/(2-\gamma))]^2}$.

By applying lemma $\ref{lemma:vmax}$ and $\max_{1\leq k\leq K} M_k \leq \frac{\psi(B^*,\Sigma^{(1:K)})}{n}(1+\delta'')$, i.e., equation \eqref{eq:maxmk} in Appendix $\ref{app:boundsigma}$, we obtain
\begin{flalign}
P\left[ |V_{max}-\mE V_{max}|\geq\delta \Big|X_S^{(1:K)}, \uW^{(1:K)} \right]&\leq 4\exp{\left( -\frac{n\delta^2}{2\rho_u\left(\Sigma^{(1:K)}_{S^cS^c|S}\right)\psi(B^*,\Sigma^{(1:K)})(1+\delta'')} \right)} \nn \\
&\leq 4\exp{\left( -\frac{n\delta^2 C_{min}}{2\rho_u\left(\Sigma^{(1:K)}_{S^cS^c|S}\right)s(1+\delta'')} \right)}
\end{flalign}
which implies $V_{max} > 2-\gamma$ with high probability.

Therefore, $\left\| V_{S^c} \right\|_{l_\infty/l_2}>1$ holds with probability larger than
\begin{flalign}
1&-2\exp\left(-\frac{s}{2}\right)-4\exp\left(-\frac{n}{2}\left(\frac{1}{4}-\sqrt{\frac{s}{n}}\right)_+^2\right)-4\exp{\left( -\frac{n\delta^2 C_{min}}{2\rho_u\left(\Sigma^{(1:K)}_{S^cS^c|S}\right)s(1+\delta'')} \right)}    \nn \\
&-\exp{\left( -5(n-s)\left[ 1-2\sqrt{\frac{1}{5}} \right] \right)}, \nn
\end{flalign}
which concludes the proof.

\section{Conclusions}

In this paper, we have investigated the Gaussian MVMR linear regression model. We have characterized sufficient and necessary conditions under which the $l_1/l_2$-regularized multi-task Lasso guarantees successful recovery of the support union of $K$ linear regression vectors. The two conditions are characterized by a threshold and hence are tight in the order sense. Our numerical results have demonstrated the advantage of joint recovery of the support union compared to using single-task Lasso to recover the support set of each task individually. Further studying the MVMR model under other block-constrains is an interesting topic in the future. Applications of the approach here to structure learning problems based on real data sets such as social network data are also interesting.

\vspace{1.6cm}

\appendix

\noindent {\Large \textbf{Appendix}}

\section{Bounds on $\psi(B^*,\Sigma^{(1:K)})$}\label{app:psi}

We first derive an upper bound on $\psi(B^*,\Sigma^{(1:K)})$ as follows:
\begin{flalign}
\psi(B^*,\Sigma^{(1:K)})&=\max_{1\leq k\leq K}\uZ_{Sk}^{*T}\left( \Sigma_{SS}^{(k)} \right)^{-1}\uZ_{Sk}^{*} \nn \\
&\leq \sum_{k=1}^K \uZ_{Sk}^{*T}\left( \Sigma_{SS}^{(k)} \right)^{-1}\uZ_{Sk}^{*} \nn \\
&\leq \sum_{k=1}^K \|\uZ_{Sk}^{*}\|_{l_2}^2\vertiii{\left( \Sigma_{SS}^{(k)} \right)^{-1}}_2 \nn \\
&\leq \frac{s}{C_{min}}.
\end{flalign}

We then derive a lower bound on $\psi(B^*,\Sigma^{(1:K)})$ as follows:
\begin{flalign}
\psi(B^*,\Sigma^{(1:K)})&=\max_{1\leq k\leq K}\uZ_{Sk}^{*T}\left( \Sigma_{SS}^{(k)} \right)^{-1}\uZ_{Sk}^{*} \nn \\
&\geq \frac{1}{K}\sum_{k=1}^K \uZ_{Sk}^{*T}\left( \Sigma_{SS}^{(k)} \right)^{-1}\uZ_{Sk}^{*} \nn \\
&\geq \frac{1}{K}\sum_{k=1}^K \|\uZ_{Sk}^{*}\|_{l_2}^2\cdot \min_{\overrightarrow{x}}\frac{{\overrightarrow{x}}^T\left( \Sigma_{SS}^{(k)} \right)^{-1}{\overrightarrow{x}}}{\left\|{\overrightarrow{x}}\right\|^2_{l_2}} \nn \\
&\geq \frac{s}{KC_{max}}
\end{flalign}
Therefore, $\psi(B^*,\Sigma^{(1:K)})$ is of the order of $O(s)$.

\section{Proof of Lemma $\ref{lemma:zeroset}$}\label{app:solution}

Suppose $\tB$ is another optimal solution to the problem given in \eqref{eq:optprob}, then we have
\begin{flalign}\label{eq:kkthb}
 f(\hB)+\lambda_n\|\hB\|_{l_1/l_2}= f(\tB)+\lambda_n\|\tB\|_{l_1/l_2},
\end{flalign}
where $f(B)=\frac{1}{2n}\sum_{k=1}^{K} \left\|\uY^{(k)}-X^{(k)} \ubeta^{(k)}\right\|_2^2$. It is clear that
\begin{flalign}\label{eq:hbl1l2}
\|\hB\|_{l_1/l_2}=\sum_{j=1}^p {\hZ_j\hB_j^T} ,
\end{flalign}
where $\hZ_j$ is the $j$th row of $\hZ$ and $\hB_j$ is the $j$th row of $\hB$. We substitute \eqref{eq:hbl1l2} into \eqref{eq:kkthb} and obtain
\[ f(\hB)+\lambda_n\sum_{j=1}^p {\hZ_j\hB_j^T} = f(\tB)+\lambda_n\|\tB\|_{l_1/l_2}. \]
We then subtract $\lambda_n\sum_{j=1}^p {\hZ_j\tB_j^T} $ from both sides of the above equation, and move $f(\tB)$ to the left-hand-side (LHS) to obtain
\begin{flalign}\label{eq:convex}
 &f(\hB)+\lambda_n\sum_{j=1}^p {\hZ_j(\hB_j^T-\tB_j^T)}-f(\tB) =\lambda_n\|\tB\|_{l_1/l_2}-\lambda_n\sum_{j=1}^p {\hZ_j\tB_j^T}.
\end{flalign}
We further substitute the KKT condition $\nabla_Bf(\hB)+\lambda_n\hZ=0$ into \eqref{eq:convex}, and obtain
\begin{flalign}
&f(\hB)+\sum_{j=1}^{p}\nabla_{B_j}f(\hB)(\tB_j^T-\hB_j^T)-f(\tB) =\lambda_n\|\tB\|_{l_1/l_2}-\lambda_n\sum_{j=1}^p {\hZ_j\tB_j^T}
\end{flalign}
Due to the convexity of $f(B)$, the LHS of the above equation is less than or equal to 0. Hence, we have
\[ \|\tB\|_{l_1/l_2}\leq\sum_{j=1}^p\hZ_j\tB_j^T . \]
Since $\sum_{j=1}^p\left\|\tB_j\right\|_{l_2}\ge\sum_{j=1}^p\hZ_j\tB_j^T$, we obtain
\[\sum_{j=1}^p\left\|\tB_j\right\|_{l_2}=\sum_{j=1}^p\hZ_j\tB_j^T.\]
Based on the assumption of the lemma, $\left\|\hZ_j\right\|_{l_2}<1$ if $j\in\Omega$. Therefore, $\left\|\tB_j\right\|_{l_2}=0$ for $j \in \Omega$.

\section{Derivation of $\hZ_{S^c}$}\label{app:hZSc}

We write the function $f(B)$ as
\begin{flalign}
f(B) &=\frac{1}{2n}\sum_{k=1}^{K} \left\|\uY^{(k)}-(X_S^{(k)},X_{S^c}^{(k)})\left( \begin{array}{c} \ubeta_S^{(k)} \\ \ubeta_{S^c}^{(k)} \end{array}\right)\right\|^2  \nn \\
&= \frac{1}{2n}\sum_{k=1}^{K} \left\|X_S^{(k)}{\ubeta_S^{(k)}}^*+\uW^{\pk}-X_S^{(k)}\ubeta_S^{(k)}- X_{S^c}^{(k)}\ubeta_{S^c}^{(k)}\right\|^2
\end{flalign}
and take partial derivative over components of $B$ to obtain
\begin{flalign}
\frac{\partial f(B)}{\partial B_{jk}}=&-\frac{1}{n}{\uX_j^{\pk}}^T\left( X_S^{(k)}{\ubeta_S^{(k)}}^*+\uW^{\pk}-X_S^{(k)}\ubeta_S^{(k)}- X_{S^c}^{(k)}\ubeta_{S^c}^{(k)} \right), \nn
\end{flalign}
where $\uX_j$ denotes the $j$th column of the matrix $X$. Hence, $\hB_S$ satisfies
\[ -\frac{1}{n}{X_S^{\pk}}^T \left( X_S^{(k)}{\ubeta_S^{(k)}}^*+\uW^{\pk}-X_S^{(k)}\hubeta_S^{(k)} \right)+\lambda_n \huZ_{Sk}=0 \]
for $k=1,\ldots,K$ ,where $\huZ_{Sk}$ denotes the $k$th column of $\hZ$ with row indices in the set $S$, and $\hZ_S \in \partial \|\hB_S\|_{l_1/l_2} $. Furthermore, $\hZ_{S^c}$ satisfies
\begin{flalign}
&-\frac{1}{n}{X_{S^c}^{\pk}}^T \left( X_S^{(k)}{\ubeta_S^{(k)}}^*+\uW^{\pk}-X_S^{(k)}\hubeta_S^{(k)} \right) +\lambda_n \huZ_{S^ck}=0  \nn
\end{flalign}
for $k=1,\ldots,K$, where $\hZ_{S^c}\in \partial \left\|\hB_{S^c}\right\|_{l_1/l_2}$. As we introduce the notations $\hSigma_{SS}^{\pk}=\frac{1}{n} {X_S^{\pk}}^TX_S^{\pk}$ and $\hSigma_{S^cS}^{\pk}=\frac{1}{n} {X_{S^c}^{\pk}}^TX_S^{\pk}$, the above two equations become
\begin{flalign}
& \hSigma_{SS}^{\pk}\left(\hubeta_S^{(k)}-{\ubeta_S^{(k)}}^* \right) -\frac{1}{n}{X_S^{\pk}}^T\uW^{\pk}=-\lambda_n \huZ_{Sk},  \label{eq:kkt1}   \\
& \hSigma_{S^cS}^{\pk}\left(\hubeta_S^{(k)}-{\ubeta_S^{(k)}}^* \right) -\frac{1}{n}{X_{S^c}^{\pk}}^T\uW^{\pk}=-\lambda_n \huZ_{S^ck},  \label{eq:kkt2}
\end{flalign}
for $k=1,\ldots,K$. We now solve $\hubeta_S^{(k)}-{\ubeta_S^{(k)}}^*$ from \eqref{eq:kkt1}, substitute it into \eqref{eq:kkt2}, reorganize the terms, and obtain
\begin{flalign}
\huZ_{S^c k}=&-\frac{1}{\lambda_nn}{X_{S^c}^{\pk}}^T \left(\Pi_S^{\pk}-I_n \right)\uW^{\pk}+\frac{1}{n} {X_{S^c}^{\pk}}^T{X_S^{\pk}}^T\left(\hSigma_{SS}^{\pk}\right)^{-1}\huZ_{Sk},
\end{flalign}
where $\Pi_S^{\pk}=\frac{ X_S^{\pk}\left(\hSigma_{SS}^{\pk}\right)^{-1}{X_S^{\pk}}^T}{n}$.

Hence, for $j \in S^c$,
\begin{flalign}
\hZ_{jk}=&-\frac{1}{\lambda_nn}{\uX_j^{\pk}}^T \left(\Pi_S^{\pk}-I_n \right)\uW^{\pk}+\frac{1}{n} {\uX_j^{\pk}}^TX_S^{\pk}\left(\hSigma_{SS}^{\pk}\right)^{-1}\huZ_{Sk} .
\end{flalign}

\section{Bound on $\left\|T_{j1}\right\|_{l_2}$}\label{app:t1norm}

We let $A^{\pk}=\Sigma_{S^cS}^{\pk} \left(\Sigma_{SS}^{\pk} \right)^{-1}$ and $\overrightarrow C_{Sk}=\mE(\huZ_{Sk}|X_S^{(1:K)})$, and derive
\begin{flalign}
\|T_{j1}\|_{l_2} & =\sqrt{\sum_{k=1}^K \mE^2\left(\hZ_{jk}|X_S\right)} \nn \\
& =\sqrt{\sum_{k=1}^K\left(\Sigma_{jS}^{\pk} \left(\Sigma_{SS}^{\pk} \right)^{-1}\mE\left(\huZ_{Sk}|X_S^{(1:K)}\right)\right)^2} \nn \\
& =\sqrt{\sum_{k=1}^K\left(A^{\pk}_{jS}\overrightarrow C_{Sk}\right)^2} \nn \\
& =\sqrt{\sum_{k=1}^K \sum_{a=1}^{|S|}A^{\pk}_{ja} C_{ak}\sum_{a'=1}^{|S|}A^{\pk}_{ja'}C_{a'k} } \nn \\
& \leq \sqrt{\sum_{a=1}^{|S|}\sum_{a'=1}^{|S|}\sum_{k=1}^K  \left|A^{\pk}_{ja}\right| \left| A^{\pk}_{ja'}\right| \left|C_{ak}C_{a'k}\right| } \nn \\
& \leq \sqrt{\sum_{a=1}^{|S|}\sum_{a'=1}^{|S|} \max_k \left|A^{\pk}_{ja}\right| \max_k \left|A^{\pk}_{ja'}\right| \sum_{k=1}^K \left|C_{ak}C_{a'k}\right| } \nn \\
& \leq \sqrt{\sum_{a=1}^{|S|}\sum_{a'=1}^{|S|} \max_k \left|A^{\pk}_{ja}\right| \max_k \left|A^{\pk}_{ja'}\right| \sqrt{\sum_{k=1}^K C^2_{ak}} \sqrt{\sum_{k=1}^KC^2_{a'k}} } \nn \\
& \leq \sqrt{\sum_{a=1}^{|S|}\sum_{a'=1}^{|S|} \max_k \left|A^{\pk}_{ja}\right| \max_k \left|A^{\pk}_{ja'}\right| } \nn \\
& = \sum_{a=1}^{|S|} \max_k \left|A^{\pk}_{ja}\right| = \sum_{a=1}^{|S|} A_{ja}
\end{flalign}
where $A_{ja}=\max_k\left| A_{ja}^{(k)} \right| = \max_k \left| \left( \Sigma_{S^c S}^{(k)} \left(\Sigma_{SS}^{(k)} \right)^{-1} \right)_{ja} \right|$.

\section{Bound on $\sigma^2_{jk}$}\label{app:boundsigma}

We let $\sigma_{jk}^2=\left(\Sigma_{S^cS^c|S}^{\pk} \right)_{jj} M_k$, where
\begin{flalign}
M_k =&\frac{1}{n}{\huZ_{Sk}}^T\left(\hSigma_{SS}^{\pk}\right)^{-1}\huZ_{Sk}-\frac{1}{n^2\lambda_n^2}{\uW^{\pk}}^T\left(\Pi_S^{\pk}-I_n \right)\uW^{\pk}.
\end{flalign}
We derive bounds on the term $\max_{j\in S^c}\max_{1\leq k\leq K}\sigma_{jk}^2$. We first define
\begin{flalign}
M_k^* :=&\frac{1}{n}{\uZ^*_{Sk}}^T\left(\hSigma_{SS}^{\pk}\right)^{-1}\uZ^*_{Sk}-\frac{1}{n^2\lambda_n^2}{\uW^{\pk}}^T\left(\Pi_S^{\pk}-I_n \right)\uW^{\pk}.
\end{flalign}
We also define
\[\bar M^*=\frac{1}{n}{\uZ^*_{Sk}}^T\left(\Sigma_{SS}^{\pk}\right)^{-1}\uZ^*_{Sk}+\frac{(n-s){\sigma_W^{\pk}}^2}{n^2\lambda_n^2}.\]
We then have
\[ |M_k-\bar M^*| \leq |M_k-M_k^*|+|M_k^*-\bar M^*|. \]

To find upper and lower bounds on $M_k$, we start with
\begin{flalign}
&\bar M^* - |M_k-M_k^*|-|M_k^*-\bar M^*| \leq  M_k \leq \bar M^* + |M_k-M_k^*|+|M_k^*-\bar M^*|.
\end{flalign}

We first bound
\begin{flalign}
&|M^*_k-M_k|  \nn \\
& =\frac{1}{n} \left|{\uZ^*_{Sk}}^T\left(\hSigma_{SS}^{\pk}\right)^{-1}\uZ^*_{Sk}-{\huZ_{Sk}}^T\left(\hSigma_{SS}^{\pk}\right)^{-1}\huZ_{Sk} \right| \nn \\
& =\frac{1}{n} \Big|{\uZ^*_{Sk}}^T\left(\hSigma_{SS}^{\pk}\right)^{-1}(\uZ^*_{Sk}-\huZ_{Sk}) +({\uZ^*_{Sk}}^T-{\huZ_{Sk}}^T)\left(\hSigma_{SS}^{\pk}\right)^{-1}(\uZ^*_{Sk}+(\huZ_{Sk}-\uZ_{Sk}^*)) \Big| \nn \\
& \leq \frac{1}{n}\vertiii{\left(\hSigma_{SS}^{\pk}\right)^{-1} }_2 \|\uZ^*_{Sk} -\huZ_{Sk} \|_{l_2}(\|\uZ^*_{Sk}\|_{l_2}+\|\uZ^*_{Sk}+(\huZ_{Sk}-\uZ_{Sk}^*)\|_{l_2} ) \nn \\
& \leq \frac{1}{n}\vertiii{\left(\hSigma_{SS}^{\pk}\right)^{-1} }_2 \|\uZ^*_{Sk} -\huZ_{Sk} \|_{l_2}(2\|\uZ^*_{Sk}\|_{l_2}+\|\huZ_{Sk}-\uZ_{Sk}^*\|_{l_2} ).
\end{flalign}
In the above equations,
\[\|\uZ^*_{Sk}\|_{l_2} \leq \sqrt{s}. \]
Following \eqref{eq:xtxi} in Appendix \ref{app:sigmass2}, we have
\[ \vertiii{\left(\hSigma_{SS}^{\pk}\right)^{-1} }_2 \leq \frac{2}{C_{min}}\]
with probability larger than $1-\exp\left(-\frac{n}{2}\left(\frac{1}{4}-\sqrt{\frac{s}{n}}\right)_+^2\right)$.

We also derive:
\begin{flalign}
\max_{1 \leq k\leq K} \|\uZ^*_{Sk}-\huZ_{Sk} \|_{l_2} & = \max_{1 \leq k\leq K} \sqrt{\sum_{j=1}^s (Z^*_{jk}-\hZ_{jk})^2 } \nn \\
& = \sqrt{\max_{1 \leq k\leq K}\sum_{j=1}^s (Z^*_{jk}-\hZ_{jk})^2 }
 \leq \sqrt{\sum_{k=1}^K\sum_{j=1}^s (Z^*_{jk}-\hZ_{jk})^2 } \nn \\
& =\sqrt{\sum_{j=1}^s\sum_{k=1}^K (Z^*_{jk}-\hZ_{jk})^2 }
 \leq \sqrt{s \max_{j\in S}\sum_{k=1}^K (Z^*_{jk}-\hZ_{jk})^2 } \nn \\
& = \sqrt{s}\max_{j\in S}\sqrt{ \sum_{k=1}^K (Z^*_{jk}-\hZ_{jk})^2 }
 = \sqrt{s}\|Z^*_{S}-\hZ_{S} \|_{l_\infty/l_2}
\end{flalign}

Hence, following from Lemma $\ref{lemma:delta}$, we have if $\|\Delta\|_{l_{\infty}/l_2} < \frac{1}{2}$, then
\[ \max_{1 \leq k\leq K} \|\uZ^*_{Sk}-\huZ_{Sk} \|_{l_2} \leq 4\sqrt{s}\|\Delta\|_{l_\infty/l_2}. \]

Based on the above bound, we have
\begin{flalign}
|M^*_k-M_k| & \leq \frac{2}{nC_{min}}\left( 4\sqrt{s}\left\|\Delta\right\|_{l_\infty/l_2} \right)\left( 2\sqrt{s}+4\sqrt{s}\left\|\Delta\right\|_{l_\infty/l_2} \right) \nn \\
& =\frac{16s\left\|\Delta\right\|_{l_\infty/l_2}}{nC_{min}}(1+2\left\|\Delta\right|_{l_\infty/l_2})
\end{flalign}
with probability larger than
\begin{flalign}
1-2K \exp\left( -\frac{s}{2} \right) - (4K+1)\exp\left(-\frac{n}{2}\left(\frac{1}{4}-\sqrt{\frac{s}{n}}\right)_+^2\right) -K\exp \left( -\log s +2\sqrt{2\log s} \right).
\end{flalign}

We next derive a bound on $|M_k^*-\bar M^*|$ as follows.
\begin{flalign}
& |M_k^*-\bar M^*| \nn \\
\leq & \frac{1}{n}\left|\uZ_{Sk}^{*T}\left(\left(\hSigma_{SS}^{(k)}\right)^{-1}-\left(\Sigma_{SS}^{(k)}\right)^{-1}\right)\uZ_{Sk}^{*}\right| +\frac{1}{n^2\lambda_n^2}\left|\uW^{(k)T}(I_n-\Pi_s^{(k)})\uW^{(k)}-(n-s){\sigma_W^{(k)}}^2\right| \nn \\
 \leq & \frac{1}{n} \left\|\uZ_{Sk}^{*T}\right\|_{l_2}^2\vertiii{\left(\hSigma_{SS}^{(k)}\right)^{-1}-\left(\Sigma_{SS}^{(k)}\right)^{-1}}_2+\frac{1}{n^2\lambda_n^2}\left|\uW^{(k)T}(I_n-\Pi_s^{(k)})\uW^{(k)}-(n-s){\sigma_W^{(k)}}^2\right| . \nn
\end{flalign}
In the above equation,
\[\left\|\uZ_{Sk}^{*T}\right\|_{l_2}^2 \leq s.\]
Following \eqref{eq:xtximsi} in Appendix $\ref{app:sigmass2}$, we have
\[\vertiii{\left(\hSigma_{SS}^{(k)}\right)^{-1}-\left(\Sigma_{SS}^{(k)}\right)^{-1}}_2\leq \frac{12}{C_{min}}\sqrt{\frac{s}{n}}\]
with probability larger than $1-2\exp\left(-\frac{s}{2}\right)-3\exp\left(-\frac{n}{2}\left(\frac{1}{4}-\sqrt{\frac{s}{n}}\right)_+^2\right)$.

We next bound the term $\left|\uW^{(k)T}(I_n-\Pi_S^{(k)})\uW^{(k)}-(n-s){\sigma_W^{(k)}}^2\right|$. Since $\Pi_S^{(k)}$ is a projection matrix, eigenvalues of $I_n-\Pi_S^{(k)}$ can only be 1 or 0. Thus, $Tr(I_n-\Pi_S^{(k)})=(n-s)$ implies that if we decompose $I_n-\Pi_S^{(k)}$ into $U^T\Lambda U$ with $U^TU=I$, then $\Lambda$ has $(n-s)$ of ``1" and $s$ of ``0". Moreover, $U\uW^{(k)}$ is a Gaussian vector with zero mean, and $\mE\left( U\uW^{(k)}\uW^{(k)T}U^T \right)={\sigma_W^{(k)}}^2I_n$. Therefore, we conclude that
\[ U\uW^{(k)}\overset{d.}{=}\uW^{(k)} \]
\[ \uW^{(k)T}U^T\Lambda U\uW^{(k)}\overset{d.}{=}\uW^{(k)T}\Lambda \uW^{(k)}\overset{d.}{=}H{\sigma_W^{(k)}}^2 \]
where $H\sim\chi_{(n-s)}^2$. We now consider the term
\begin{flalign}
&\left|\uW^{(k)T}(I_n-\Pi_S^{(k)})\uW^{(k)}-(n-s){\sigma_W^{(k)}}^2\right|   \nn \\
&=\left|\uW^{(k)T}U^T\Lambda U\uW^{(k)}-(n-s){\sigma_W^{(k)}}^2\right|  \nn \\
&=\left|\uW^{(k)T}\Lambda \uW^{(k)}-(n-s){\sigma_W^{(k)}}^2\right|  \nn \\
&=\left|H{\sigma_W^{(k)}}^2-(n-s){\sigma_W^{(k)}}^2\right|.
\end{flalign}
We derive the probability of the following event:
\begin{flalign}
&P\left( \left|H{\sigma_W^{(k)}}^2-(n-s){\sigma_W^{(k)}}^2\right|\leq 9(n-s){\sigma_W^{(k)}}^2 \right) \nn \\
&=P\left( \{ H<10(n-s) \}\cap \{ H>-8(n-s) \} \right)  \nn \\
&=P\left( H<10(n-s) \right).
\end{flalign}
Following from Lemma $\ref{lemma:chi2}$, we have
\[ P(H\geq 10(n-s))\leq\exp{\left( -5(n-s)\left[ 1-2\sqrt{\frac{1}{5}} \right] \right)} .\]
It then follows that
\[ \left|\uW^{(k)T}(I_n-\Pi_S^{(k)})\uW^{(k)}-(n-s){\sigma_W^{(k)}}^2\right|\leq 9(n-s){\sigma_W^{(k)}}^2 \]
with probability larger than
\[ 1-\exp{\left( -5(n-s)\left[ 1-2\sqrt{\frac{1}{5}} \right] \right)}. \]

To summarize,
\[|M_k^*-\bar M^*| \leq \frac{12}{C_{min}}{\left(\frac{s}{n}\right)}^{\frac{3}{2}}+\frac{9(n-s){\sigma_W^{(k)}}^2}{n^2\lambda_n^2}\]
with probability larger than
\begin{flalign}
1&-2\exp\left(-\frac{s}{2}\right)-3\exp\left(-\frac{n}{2}\left(\frac{1}{4}-\sqrt{\frac{s}{n}}\right)_+^2\right)-\exp{\left( -5(n-s)\left[ 1-2\sqrt{\frac{1}{5}} \right] \right)}. \nn
\end{flalign}
Therefore,
\begin{flalign}
|M_k-\bar M^*|& \leq |M_k-M_k^*|+|M_k^*-\bar M^*| \nn \\
& \leq \frac{16s\left\|\Delta\right\|_{l_\infty/l_2}}{nC_{min}}(1+2\left\|\Delta\right\|_{l_\infty/l_2})+\frac{12}{C_{min}}{\left(\frac{s}{n}\right)}^{\frac{3}{2}}+\frac{9(n-s){\sigma_W}^2}{n^2\lambda_n^2}
\end{flalign}
with high probability.

To simplify the result, we define the following quantity
\begin{flalign}
\Gamma:=&\frac{16s\left\|\Delta\right\|_{l_\infty/l_2}}{nC_{min}}(1+2\left\|\Delta\right|_{l_\infty/l_2})+\frac{12}{C_{min}}{\left(\frac{s}{n}\right)}^{\frac{3}{2}}+\frac{10(n-s){\sigma_W}^2}{n^2\lambda_n^2}
\end{flalign}
and our bounds on $M_k$ can be expressed as
\begin{flalign}
\frac{1}{n}{\uZ^*_{Sk}}^T & \left(\Sigma_{SS}^{\pk}\right)^{-1}\uZ^*_{Sk}-\Gamma  \leq M_k \leq \frac{1}{n}{\uZ^*_{Sk}}^T\left(\Sigma_{SS}^{\pk}\right)^{-1}\uZ^*_{Sk}+\Gamma. \nn
\end{flalign}
Using the definition of $\psi(B^*,\Sigma^{(1:K)})$, we have
\begin{flalign}\label{eq:maxmk}
&\frac{\psi(B^*,\Sigma^{(1:K)})}{n}-\Gamma \leq \max_{1\leq k\leq K} M_k \leq \frac{\psi(B^*,\Sigma^{(1:K)})}{n}+\Gamma
\end{flalign}
with probability larger than
\begin{flalign}
1&-2(K+1)\exp\left(-\frac{s}{2}\right)-4(K+1)\exp\left(-\frac{n}{2}\left(\frac{1}{4}-\sqrt{\frac{s}{n}}\right)_+^2\right) \nn \\
&-\exp{\left( -5(n-s)\left[ 1-2\sqrt{\frac{1}{5}} \right] \right)}-K\exp \left( -\log s +2\sqrt{2\log s}\right).
\end{flalign}

\section{Bounds on Spectral Norms}\label{app:sigmass2}

In this section, we provide some useful bounds on spectral norms. Detailed proof can be found in \cite{Oboz11}.

Let $U\in R^{n\times s}$ be a random matrix with i.i.d.\ entries, and each entry has a Gaussian distribution with zero mean and unit variance.

{\em The bound for $\vertiii{\frac{1}{n}U^TU}_2$: }
\[ P\left(\vertiii{\frac{1}{n}U^TU}_2 \leq \frac{1}{2}\right) \leq \exp\left(-\frac{n}{2}\left(\frac{1}{4}-\sqrt{\frac{s}{n}}\right)_+^2\right). \]

{\em The bound for $\vertiii{\frac{1}{n}U^TU-I_{s\times s}}_2$: }
\begin{flalign}
&P\left(\vertiii{\frac{1}{n}U^TU-I_{s\times s}}_2 \geq 6\sqrt{\frac{s}{n}}\right) \leq 2\exp\left(-\frac{s}{2}\right)+2\exp\left(-\frac{n}{2}\left(\frac{1}{4}-\sqrt{\frac{s}{n}}\right)_+^2\right)
\end{flalign}

Let $X=U\sqrt{\Sigma}$ where $\Sigma \in R^{s\times s}$ is positive definite. Then $X \in R^{n\times s}$ has i.i.d.\ rows, and each row $X_i$ is a Gaussian vector with the distribution $\cN(0,\Sigma)$. Suppose the eigenvalues of $\Sigma$ are in the interval $[C_{min}, C_{max}]$, where $C_{min}$ and $C_{max}$ are both positive. We next provide the bounds on several spectral norms.

%{\em The bound for $\vertiii{\Sigma^{-1}}_2$: }$\vertiii{\Sigma^{-1}}_2\leq\frac{1}{C_{min}}$.

{\em The bound for $\vertiii{\left( \frac{U^TU}{n} \right)^{-1}}_2$:}
\begin{flalign}\label{eq:utui}
& P\left( \vertiii{\left( \frac{U^TU}{n} \right)^{-1}}_2\leq 2 \right) \geq  1-\exp\left(-\frac{n}{2}\left(\frac{1}{4}-\sqrt{\frac{s}{n}}\right)_+^2\right).
\end{flalign}
%where $\lambda_{min}(\cdot)$ to denote the smallest singular value.

{\em The bound for $\vertiii{\left(\frac{X^TX}{n}\right)^{-1}}_2$: }
\begin{flalign}\label{eq:xtxi}
P\left( \vertiii{\left(\frac{X^TX}{n}\right)^{-1}}_2 \leq \frac{2}{C_{min}} \right) \geq 1-\exp\left(-\frac{n}{2}\left(\frac{1}{4}-\sqrt{\frac{s}{n}}\right)_+^2\right).
\end{flalign}

{\em The bound for $\vertiii{\frac{X^TX}{n}-\Sigma}_2$: }
\begin{flalign}\label{eq:xtxms}
P\left( \vertiii{\frac{X^TX}{n}-\Sigma}_2 \leq 6C_{max}\sqrt{\frac{s}{n}} \right)\geq 1-2\exp\left(-\frac{s}{2}\right)-2\exp\left(-\frac{n}{2}\left(\frac{1}{4}-\sqrt{\frac{s}{n}}\right)_+^2\right).
\end{flalign}

{\em The bound for $\vertiii{\left(\frac{X^TX}{n}\right)^{-1}-\Sigma^{-1}}_2$: }
\begin{flalign}\label{eq:xtximsi}
P\left( \vertiii{\left(\frac{X^TX}{n}\right)^{-1}-\Sigma^{-1}}_2  \leq \frac{12}{C_{min}}\sqrt{\frac{s}{n}} \right) \geq 1-2\exp\left(-\frac{s}{2}\right)-3\exp\left(-\frac{n}{2}\left(\frac{1}{4}-\sqrt{\frac{s}{n}}\right)_+^2\right) .
\end{flalign}

%\section{Proof of Bounds}\label{app:infito2}
%Assume matrix $A \in R^{s\times s}$.
%\begin{flalign}
%|||A|||_{\infty} &=\max_{1\leq i\leq s}\Sigma_{j=1}^{s}|A_{ij}| \nn \\
%&= \max_{1\leq i\leq s}\Sigma_{j=1}^{s}A_{ij}sign(A_{ij}) \nn \\
%&= \sqrt{s}\max_{1\leq i\leq s}\Sigma_{j=1}^{s}A_{ij}\frac{sign(A_{ij})}{\sqrt{s}} \nn \\
%&= \sqrt{s}\max_{1\leq i \leq s} A_i\cdot \frac{sign(A_i^T)}{\sqrt{s}} \nn \\
%&\leq \sqrt{s}\sqrt{\sum_{i=1}^s\left(A_i\frac{sign(A_i^T)}{\sqrt{s}} \right)^2}  \nn \\
%&\leq\sqrt{s}\max_{\|x\|_{l_2}=1}{\|A_ix\|_{l_2}=1} \nn \\
%&= \sqrt{s}|||A|||_2
%\end{flalign}

\hspace{-0.5cm }\textbf{\large Acknowledgements}

The work of W. Wang and Y. Liang was supported by NSF CAREER Award CCF-10-26565 and NSF CCF-10-26566. The work of E. P. Xing was supported by NIH1R01GM087694, FA9550010247, and NIH1R01GM093156.

The authors would like to thank Dr.\ Junming Yin and Dr.\ Mladen Kolar at Carnegie Mellon University for their helpful discussions.

%%%%%%%%%%%%%%%%%%%%%%%%%%%%%%%%%%%%%%%%%%%%%%%%%%%%%%%%%%%%%%%%%%%%%%%%%%%%%%%

%\subsubsection*{References}

%%%%%%%%%%%%%%%%%%%%%%%%%%%%%%%%%%%%%%

%\renewcommand\refname{\vskip -1cm}
\bibliographystyle{agsm}
\bibliography{learning}

%%%%%%%%%%%%%%%%%%%%%%%%%%%%%%%%%%%%%%

\end{document}